\documentclass[letterpaper, 10 pt, journal, twoside]{template/IEEEtran} 
\IEEEoverridecommandlockouts

\usepackage{bmpsize}
\usepackage{amsthm}
\usepackage[linesnumbered,ruled,vlined]{algorithm2e}
\usepackage{amsmath,amssymb,enumerate}
\usepackage{mathrsfs}
\usepackage{multicol}
\usepackage{amsfonts}
\usepackage{textcomp}
\usepackage{balance}
\usepackage{multirow}
\usepackage{booktabs}
\usepackage{dblfloatfix} 
\usepackage{bbold}
\usepackage{makecell}
\usepackage{hhline} 
\usepackage{float} 
\usepackage{xparse}
\usepackage{lipsum}
\usepackage{soul} 
\usepackage{times}
\usepackage{cite}
\usepackage{balance}
\usepackage{dsfont} 
\usepackage{physics} 
\usepackage{setspace} 
\usepackage{accents} 

\usepackage[final]{graphicx} 
\usepackage[T1]{fontenc}
\usepackage[usenames,dvipsnames,svgnames,table, xcdraw]{xcolor}
\usepackage[caption=false,font=footnotesize]{subfig}
\usepackage[utf8]{inputenc}

\usepackage{hyperref}
\hypersetup{
  colorlinks=true,pdfpagemode=UseNone,citecolor=black,linkcolor=black,urlcolor=BrickRed
}

\usepackage{caption}
\SetCommentSty{mycommfont}

\SetKwInput{KwInput}{Input}                
\SetKwInput{KwOutput}{Output}              
\SetKwInput{KwInitialization}{Initialization}              

\graphicspath{
    {graphics/}
}

\newtheoremstyle{break}
  {\topsep}{\topsep}%
  {\itshape}{}%
  {\bfseries}{}%
  {\newline}{}%

\newtheorem{problem}{Problem}
\newtheorem{theorem}{Theorem}
\newtheorem{proposition}{Proposition}
\newtheorem{remark}{Remark}

\theoremstyle{break}

\newtheorem{claim}{Claim}

\newtheorem{assumptionB}{Assumption}

\newtheorem{assumptionN}{Assumption}

{
  \par\vspace{\baselineskip}\noindent
  \textbf{Assumption #1 (#2)}
  \begin{itshape}%
  \par\noindent\ignorespaces
}%
{
  \end{itshape}
  \par\vspace{\baselineskip}
} 

\newcommand{\bh}[1]{[{\textbf{\textcolor{blue}{Bruce: #1}}}]}

\definecolor{note}{rgb}{0.3,0.7,0.25}
\definecolor{rephase}{rgb}{0.15,0.7,0.15}
\definecolor{bag}{rgb}{0.5,0.5,0.0}

\newcommand{\add}[1]{\textbf{\textcolor{OliveGreen}{#1}}}
\newcommand{\addd}[1]{{\textcolor{bag}{#1}}}

\newcommand{\lidar}{LiDAR~}
\newcommand{\lidarNoSpace}{LiDAR}
\newcommand{\lidarN}{LiDAR}
\newcommand{\lidars}{LiDARs~}
\newcommand{\lidarsN}{LiDARs}
\newcommand{\velodyne}{\textit{32-Beam Velodyne ULTRA Puck LiDAR~}}
\newcommand{\velodyneN}{\textit{32-Beam Velodyne ULTRA Puck LiDAR}}

\newcommand{\sslidar}{solid-state LiDAR~}
\newcommand{\sslidarN}{solid-state LiDAR}
\newcommand{\sslidars}{solid-state LiDARs~}
\newcommand{\sslidarsN}{solid-state LiDARs}

\newcommand{\slidar}{spinning LiDAR~}
\newcommand{\slidarN}{spinning LiDAR}
\newcommand{\slidars}{spinning LiDARs~}

\newcommand{\lidart}{LiDARTag~}

\newcommand{\lind}{linearly independent}
\newcommand{\li}{linearly independent}

\newcommand{\LanDual}{Lagrangian Duality~}

\newcommand{\mm}{Min-Min~}

\newcommand{\pc}{\mathcal{PC}}

\newcommand{\BL}{\textbf{BL}}

\newcommand{\squeezeup}{\vspace{-4mm}} 

\newcommand{\comment}[1]{}

\def\realnumbers{\mathbb{R}}
\def\real{\mathbb{R}}
\def\reals{\mathbb{R}}

\newcommand{\dtheta}{\delta_\theta}
\newcommand{\dphi}{\delta_\phi}
\newcommand{\drho}{\delta_\rho}

\DeclareDocumentCommand{\RI}{ O{} O{} }{\mathcal{RI}_{#1}^{#2}}

\newcommand\inv[1]{#1\raisebox{1.15ex}{$\scriptscriptstyle-\!1$}}

\DeclareDocumentCommand{\td}{ O{} }{\tilde{#1}}
\newcommand{\transpose}{\mathsf{T}}

\newcommand{\SO}{\mathrm{SO}}

\newcommand{\SE}{\mathrm{SE}}
\newcommand{\Sim}{\mathrm{Sim}}

\newcommand{\vect}[1]{\mathrm{vec}(#1)}

\DeclareDocumentCommand{\asin}{ O{} }{\sin^{-1}(#1)}
\DeclareDocumentCommand{\acos}{ O{} }{\cos^{-1}(#1)}
\DeclareDocumentCommand{\atan}{ O{} }{\tan^{-1}(#1)}

\DeclareDocumentCommand{\vector}{ O{} }{\mathrm{vec}(#1)}
\DeclareDocumentCommand{\zeros}{ O{} }{\textbf{0}_{#1}}
\DeclareDocumentCommand{\pre}{ O{} O{} }{{}_{#1}^{#2}}

\DeclareMathOperator*{\argmin}{arg\,min}

\newcommand{\Scal}{\mathcal{S}}

\newcommand{\Xcal}{\mathcal{X}}

\DeclareDocumentCommand{\A}{ O{} O{} }{\textbf{A}_{#1}^{#2}}
\DeclareDocumentCommand{\G}{ O{} O{} }{\textbf{G}_{#1}^{#2}}
\DeclareDocumentCommand{\H}{ O{} O{} }{\textbf{H}_{#1}^{#2}}
\DeclareDocumentCommand{\I}{ O{} O{} }{\textbf{I}_{#1}^{#2}}
\DeclareDocumentCommand{\L}{ O{} O{} }{\textbf{L}_{#1}^{#2}}
\DeclareDocumentCommand{\M}{ O{} O{} }{\textbf{M}_{#1}^{#2}}
\DeclareDocumentCommand{\N}{ O{} O{} }{\textbf{N}_{#1}^{#2}}
\DeclareDocumentCommand{\O}{ O{} O{} }{\textbf{O}_{#1}^{#2}}
\DeclareDocumentCommand{\P}{ O{} O{} }{\textbf{P}_{#1}^{#2}}
\DeclareDocumentCommand{\Q}{ O{} O{} }{\textbf{Q}_{#1}^{#2}}
\DeclareDocumentCommand{\R}{ O{} O{} }{\textbf{R}_{#1}^{#2}}
\DeclareDocumentCommand{\T}{ O{} O{} }{\textbf{T}_{#1}^{#2}}
\DeclareDocumentCommand{\U}{ O{} O{} }{\textbf{U}_{#1}^{#2}}
\DeclareDocumentCommand{\V}{ O{} O{} }{\textbf{V}_{#1}^{#2}}
\DeclareDocumentCommand{\X}{ O{} O{} }{\textbf{X}_{#1}^{#2}}
\DeclareDocumentCommand{\Y}{ O{} O{} }{\textbf{Y}_{#1}^{#2}}
\DeclareDocumentCommand{\Z}{ O{} O{} }{\textbf{Z}_{#1}^{#2}}

\DeclareDocumentCommand{\e}{ O{} O{} }{\textbf{e}_{#1}^{#2}}
\DeclareDocumentCommand{\n}{ O{} O{} }{\textbf{n}_{#1}^{#2}}
\DeclareDocumentCommand{\o}{ O{} O{} }{\textbf{o}_{#1}^{#2}}
\DeclareDocumentCommand{\t}{ O{} O{} }{\textbf{t}_{#1}^{#2}}
\DeclareDocumentCommand{\p}{ O{} O{} }{\textbf{p}_{#1}^{#2}}
\DeclareDocumentCommand{\q}{ O{} O{} }{\textbf{q}_{#1}^{#2}}
\DeclareDocumentCommand{\r}{ O{} O{} }{\textbf{r}_{#1}^{#2}}
\DeclareDocumentCommand{\u}{ O{} O{} }{\textbf{u}_{#1}^{#2}}
\DeclareDocumentCommand{\v}{ O{} O{} }{\textbf{v}_{#1}^{#2}}
\DeclareDocumentCommand{\x}{ O{} O{} }{\textbf{x}_{#1}^{#2}}

\begin{document}

\title{Global Unifying Intrinsic Calibration for~Spinning~and~Solid-State~LiDARs}
\author{Jiunn-Kai Huang, Chenxi Feng, Madhav Achar, Maani Ghaffari, and Jessy W. Grizzle
\thanks{The authors are with the Robotics Institute, University of Michigan, Ann Arbor, MI 48109, USA. \texttt{\{bjhuang, chenxif, achar, maanigj, grizzle\}@umich.edu}.} }


\maketitle
\thispagestyle{empty}
\pagestyle{plain}

\begin{abstract}
Sensor calibration, which can be intrinsic or extrinsic, is an essential step to
achieve the measurement accuracy required for modern perception and navigation
systems deployed on autonomous robots. Intrinsic calibration models for
spinning LiDARs have been based on hypothesized physical mechanisms,
resulting in anywhere from three to ten parameters to be estimated from data, while
no phenomenological models have yet been proposed for solid-state LiDARs. Instead of
going down that road, we propose to abstract away from the physics of a LiDAR type (spinning vs. solid-state, for example) and focus on the point cloud's spatial geometry generated by the sensor. By modeling the calibration parameters as an
element of a matrix Lie Group, we achieve a unifying view of calibration for
different types of LiDARs. We further prove mathematically that the proposed model is
well-constrained (has a unique answer) given four appropriately orientated
targets. The proof provides a guideline for target positioning in the form of a
tetrahedron. Moreover, an existing semi-definite programming global solver for SE(3) can be modified to efficiently compute the optimal calibration parameters. For solid-state LiDARs, we illustrate how the method works in simulation. For
spinning LiDARs, we show with experimental data that the proposed matrix Lie Group
model performs equally well as physics-based models in terms of reducing the point-to-plane 
distance while being more robust to noise.
\end{abstract}


%

\section{ Introduction}
\label{sec:intro}

Camera and Light Detection And Ranging (\lidarN) sensors and supporting software are common system elements in current
robotic and autonomous systems. For real-world operation, such systems' performance depends on the quality of intrinsic and extrinsic calibration parameters. Intrinsic calibration of a sensor is the process of ensuring that
obtained measurements are meaningful and valid. Intrinsic calibration of cameras is
relatively mature~\cite{zhang2000flexible, liebowitz1998metric, stein1995accurate} and many open-source packages are available~\cite{maye2013self, opencv, ros}. Similarly, intrinsic calibration is also required for \lidarsN; otherwise, the obtained point clouds are potentially inaccurate. However, even though \slidars have been on the market for over ten years, no publication supported by open-source calibration software is currently available, forcing academic researchers to either spend time developing their own solutions or to ignore concerns about the accuracy of their \lidar point clouds. 


\begin{figure}[t]
\centering
\includegraphics[width=1\columnwidth]{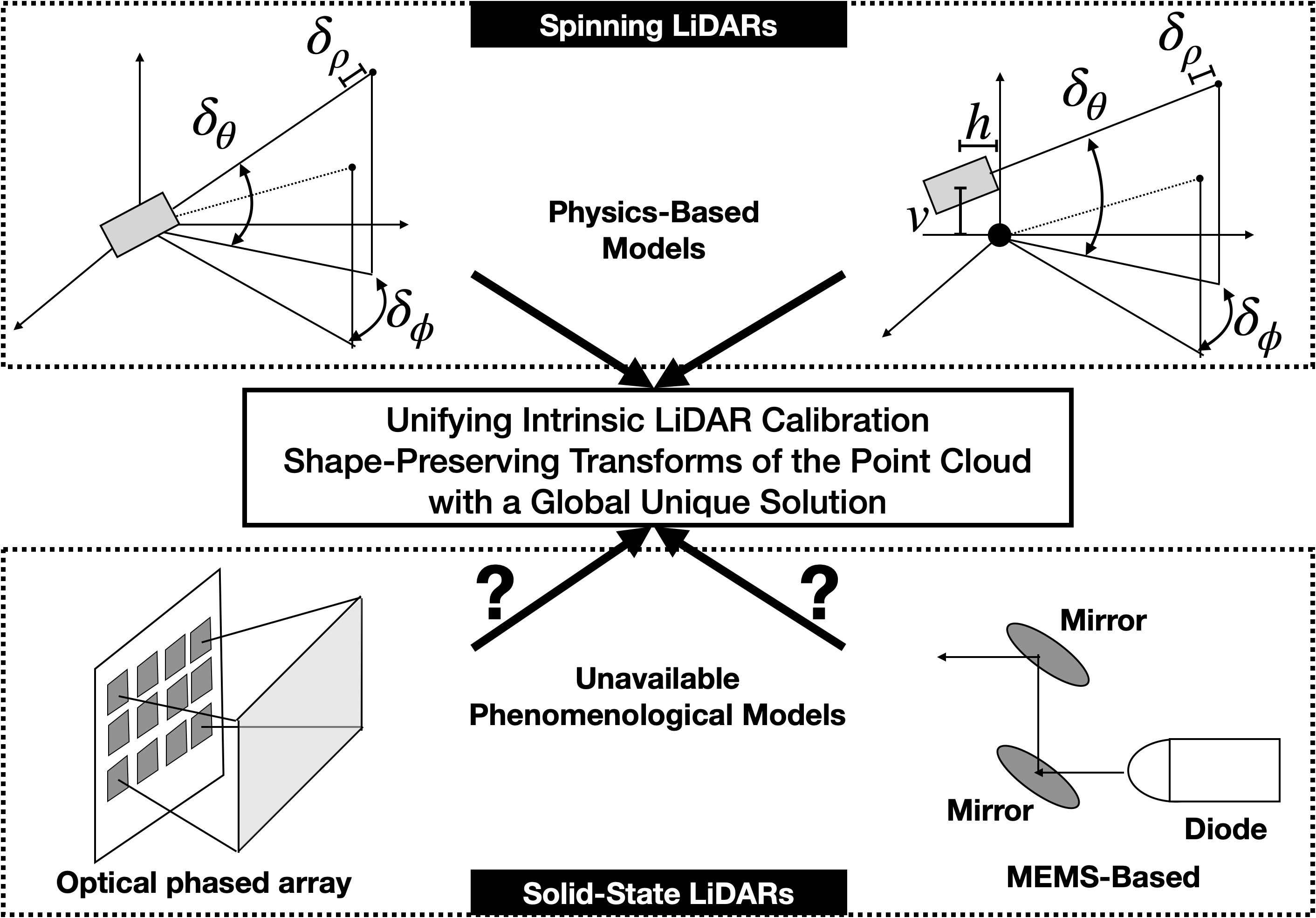}
\caption{Is it possible to calibrate a \lidar without modeling the physical mechanism itself? This paper says YES it is and it allows a unifying means to calibrate \lidars of all types. Top left: shows a simple physics-based calibration model for a spinning
    \lidarN. Top right: shows a similar model with more parameters, also for a
    spinning \lidarN. In the literature, the number of parameters can vary from three
    to ten. Bottom left and bottom right show, respectively, an optical phased array
    \sslidar and a MEMS-based \sslidarN. There are many design options and many
    manufacturing steps to fabricate these \lidarsN. This paper moves the attention to the
    geometry of the \lidarN's point cloud and proposes that calibration can be
    achieved via shape-preserving transformations.
 } 
\label{fig:FirstImage}
\squeezeup
\end{figure}

    Recently, another type of \lidarN, \sslidarN, is being brought to the market.
    Industry is turning to \sslidars because they are small, quiet and produce less
vibration. Due to their small size, they can be nicely integrated with headlights or carried by mobile robots such as drones. Moreover, \sslidars are more reliable especially for automotive applications due to the absence of the rotating mechanism. Because \slidars have a relatively
simple working mechanism, it has been possible to propose physics-based calibration models~\cite{mirzaei20123d,pandey2010extrinsic,glennie2010static,
nouira2016point,chan2013temporal,atanacio2011lidar,
muhammad2010calibration,bergelt2017improving},
see Sec.~\ref{sec:SpinningLiDAR}. For \sslidarsN, however, multiple types and design options are being proposed (see Sec.~\ref{sec:SolidStateLiDAR}), and thus it is going to be difficult to hand-craft a calibration model for each design of a \sslidarN. We ask, is it necessary to do this? Are type-specific calibration models even the right way to approach the problem?

In this work, we propose a unifying view of calibration for different types of
\lidarsN, as shown in Fig.~\ref{fig:FirstImage}: we tie the calibration process to the data produced by the sensor, and not the sensor itself, by applying shape-preserving transformations to a \lidarN's point cloud. Additionally, we prove
mathematically that the proposed calibration model is well-constrained (i.e., a
unique answer exists), whereas existing physics-based calibration models do not guarantee a unique set of calibration parameters. The uniqueness proof for the proposed method provides a guideline for target
positioning. In addition, uniqueness of the parameters is necessary if one is to investigate a globally convergent parameter solver.

\comment{
In this work, we are concerned with the problem of
multi-beam \lidar intrinsic calibration. We focus on narrow-pulsed time of flight
(ToF) method: spinning \lidars and \sslidarsN, as shown in Fig.~\ref{fig:FirstImage}.
}

\squeezeup
\subsection{Our Contributions}
\label{sec:Contributions}
\begin{enumerate}
     \item We propose a model for intrinsic \lidar calibration that is independent of the underlying technology. The new error model is based on
         shape-preserving transformations\footnote{Called \textit{similarity
         transformations} in projective geometry.}, which is a rigid body
         transformation plus a scale factor. Currently, intrinsic calibration models
         have been published for spinning LiDARs, while for \sslidarsN, the calibration models used by manufacturers have not found their way into the published literature. Current calibration models are tied to
         postulated physical mechanisms for how a \lidar functions, resulting in
         anywhere from three to ten parameters to be estimated from data. Instead, we
         \textit{abstract away from the physics of a \lidar type} (spinning head vs.
         solid-state, for example) and focus on the spatial geometry of the point
         cloud generated by the sensor. \textit{This leads to a unifying view of
         calibration} and results in an error model that can be represented as the
         action of the seven-dimensional matrix Lie group, $\Sim(3)$, on subsets of
         the measurements returned by the \lidarNoSpace.

     \item One may suspect that our proposed model is over parameterized. We, therefore, mathematically prove that given four targets with appropriate orientations, the proposed model is well-constrained (i.e., only one answer exists). The proof provides a guideline for target placement, which is an oriented tetrahedron.



     \item We give a globally convergent means to determine the calibration parameters. It utilizes an alternating optimization technique to solve the
         calibration problem, where the scaling parameter is determined by the
     bisection method, and the rest of the parameters ($\SE(3)$) are formulated as a
     Quadratically Constrained Quadratic Program (QCQP) where the Lagrangian dual
     relaxation is used. The relaxed problem on $\SE(3)$ becomes a Semidefinite Program (SDP) and
     convex~\cite{tron2015inclusion, carlone2015lagrangian, olsson2008solving,
     briales2017convex}. Therefore, the calibration problem can be solved globally and
     efficiently by off-the-shelf solvers~\cite{grant2014cvx}.

     \item We provide a MATLAB-based simulator for intrinsic \lidar calibration and
        validation that allows one to study the effects of target placement and the
        effectiveness of various calibration methods.

     \item In simulations, we validate that positioning the four
         targets as a tetrahedron can calibrate equally well a
         spinning \lidar and a \sslidarN. In experiments on a spinning \lidarNoSpace,
         we show that the proposed method works as well as two physics-based models in terms of Point-to-Plane (P2P) distance, while it is more robust when we induce systematic errors to the experimental measurements.

     \item We open-source all of the related software for intrinsic
        calibration, including implementations of the baseline methods, the simulator,
        the data sets and the solver; see~\cite{githubFileIntrinsic,
        githubFileLiDARSimulator, githubFileGlobalSim3Solver}. 

\end{enumerate}

\squeezeup
\section{Related Work}
This section overviews related work as well as spinning and \sslidarsN. 

\squeezeup
\subsection{Spinning \lidar}
\label{sec:SpinningLiDAR}
A spinning \lidar as used in this paper consists of $K$ laser emitter-receiver pairs
mounted at various positions and elevation angles on a rotating head, as shown in
Fig.~\ref{fig:FirstImage}. The lasers are pulsed at a fixed frequency. Each
revolution of the rotating head sweeps out a cone in space, called a \lidar scan,
composed of discrete return points traced out by each pulsed emitter-receiver. The
points associated with a single laser beam are called a ring. 

The measurements are typically viewed as being made in a spherical coordinate system.
The range measurement is made by estimating time-of-flight (ToF) through a clocking system on a
circuit board, the azimuth is estimated from the position of the rotating head, and
the elevation angle of a beam is determined by its mounting angle on the rotating
head.


Existing error models for \lidars are phenomenological in nature, that is, they are
tied to physical explanations of a sensor's design and/or
operation~\cite{mirzaei20123d,pandey2010extrinsic,glennie2010static,
nouira2016point,chan2013temporal,atanacio2011lidar,
muhammad2010calibration,bergelt2017improving}. Broadly speaking, two common sources of
uncertainty considered in spinning \lidar models are (1) errors in the
clocking systems, (one for ToF and one for pulsing the lasers) and (2) mechanical
errors associated with positioning of the beams on the spinning \lidarN. Less commonly considered are affects are due to ambient temperature or target size. Some authors
overlook that the measurements are actually made in a local coordinate frame for each
emitter-receiver pair, and are then transformed to a single ``global'' frame for the
\lidar unit, resulting in models with three to six parameters
\cite{pandey2010extrinsic, muhammad2010calibration, glennie2010static}, while others
take this into account and use as many as ten parameters~\cite{bergelt2017improving}.
Moreover, none of these models provides a guideline for target placement so that
the calibration model is well-constrained (i.e., a unique solution exists), see
Sec.~\ref{sec:TargetPlacementGuideline}.

\squeezeup
\subsection{Solid-State \lidar}
\label{sec:SolidStateLiDAR}
In comparison  to spinning \lidarsN, \sslidars are smaller, quieter, lower-cost, and
produce less vibration. Several types of \sslidar are coming to the market, including
those based on (1) Optical Phased Array (OPA) \cite{frederiksen2020lidar,
poulton2019long, poulton2018high}, and (2) Micro-Electro-Mechanical System (MEMS)
mirrors\footnote{Some authors argue MEMS-based \sslidars are not truly solid-state
devices due to the movable mirrors; the mirrors, however, are tiny compared to a
spinning \lidarN.}\cite{yoo2018mems, pelz2020matrix, druml20181d, lee2015optical,
wang2020mems, garcia2020geometric}. 

OPA-based \sslidars function similarly to a phased array in antenna theory. A single
OPA contains a significant number of emitters patterned on a chip. By varying the
phase shift (i.e., time delay) of each emitter, the direction of the resulting
wave-front or the spread angle can be adjusted. In particular, the sensor is able to
zoom in or out of an object on command. On the other hand, MEMS-based \sslidars use
only one or a few emitters, with the beams guided by a mirror system regulated by a
sophisticated controller. The controller seeks to steer the beam (or beams) in a
scanning pattern, such as a zig-zag \cite{yoo2018mems}. 

To perfectly align OPA emitters or MEMS mirrors is challenging. It is also hard to
maintain high accuracy under different temperatures and weather conditions. Thus, how
to properly calibrate a \sslidar is a critical problem.

The variety of \sslidars raises the question: is it necessary, or even possible, to
design/create parameterized models for every type of \sslidarN? This paper seeks to
present a unifying view of \lidar calibration by abstracting away from the physics of
a \lidar type (spinning vs solid state, for example), and focusing instead on the
spatial geometry of the point cloud generated by the sensor. This leads to a unifying
view of calibration.





\section{Proposed \lidar Intrinsic Model and Analysis}
\label{sec:IntrinsicCalibration}
The proposed calibration parameters are modeled as an element of the similarity Lie
group $\Sim(3)$, which can be applied to both
\slidars and \sslidarsN. An element of the group is an isometry composed with an additional
isotropic scaling, i.e., a rigid-body transformation with scaling. We note that this
is the most general form of of a shape-preserving transformation in 3D space (i.e. preserves ratios of lengths and angles). We also note that \eqref{eq:SimpleCalibration} and \eqref{eq:ComplexCalibration} are not shape preserving. It will be important to see in real data if our assumption of shape preservation is general enough to provide useful calibrations. 

An element of the group is given by
\begin{equation}
\label{eq:sim3tf}
        \H = \begin{bmatrix}
         s \R & \v \\
         \zeros & 1
     \end{bmatrix} \in \Sim(3),
\end{equation}
where $\R \in \SO(3)$ is the 3D rotation matrix, $\v \in \mathbb{R}^3$ is the
translation vector, and $s \in \mathbb{R}^+$ is the scale parameter. In particular,
an element has seven degrees of freedom and the action of $\Sim(3)$ on $\mathbb{R}^3$
is $\H \cdot \x = s\R \x + \v$, where $\x \in \reals^3$ and where $\H \cdot \x$ uses homogeneous coordinates while $s\R \x + \v$ uses Cartesian coordinates. From hereon, wherever convenient, we abuse notation by passing from regular coordinates to homogeneous coordinates without noting the distinction.

\subsection{Point-to-Plane Distance}
\label{sec:KernelizedP2P}
Let $\n[t]$ be a unit normal vector of a planar target and let $\p[0, t]$ be a fixed
point on the target. Let $\Xcal = \{ \x[i] | i=1,\dots,M~\text{and}~\x[i] \in
\mathbb{R}^3 \}$ denote a collection of \lidar returns on the target, and $M$ be
the number of points in the collection. Given the collection of \lidar returns $\Xcal$,
the cost of the intrinsic calibration problem is commonly defined by the
P2P distance:
\begin{equation}
    \label{eq:P2PDistance}
    J_t := \sum_{i=1}^M J_i(\x[i], \n[t], \p[0,t]):= \sum_{i=1}^M
   \left( \n[t][\transpose] \left(\x[i] - \p[0,t] \right)\right)^2,
\end{equation}
where $\n[t][\transpose] \left(\x[i] - \p[0,t] \right)$ is the orthogonal projection
of the measurement onto the normal vector. Then the calibration problem can be formulated as follows
\begin{problem}
\label{prob:intrinsic_calibration}
For a given collection of points $\Xcal_t = \{\x[i]\}_{i=1}^{M_t}$, possibly
from multiple targets $t\in\{1, \ldots, T\}$, we seek a similarity transformation
$\H[][\star]$ that solves
\begin{equation}
\label{eq:OptSolution}
    \min_{\H \in \Sim(3)}
   \sum_{t=1}^T\sum_{i=1}^{M_t}
    \left( \n[t][\transpose] \left( \H \cdot \x[i] - \p[0,t] \right) \right)^2.
\end{equation}
\end{problem}

\comment{
\addd{The underlying assumption of \eqref{eq:OptSolution} is the independence
of different collections of points, as each collection of points has its own
calibration parameters and there is no correlation between different collections. 
However, points from a \sslidar share the same geometry such as a wafer. Therefore,
the independence assumption is not true and leads to discontinuity of different
collection of points.}

\addd{We propose using all the \sslidar returns (possibly consisting of multiple
collections) on a target and kernelizing the points while calibrating a collection
of points. The kernelized P2P cost is defined as
\begin{problem}
\label{prob:KernelizedIntrinsicCalibration}
\add{Let the \lidar returns ($\pc_t$) on a target composite of $K$ collections and each
collection has $M_k$ points and $\pc_t = \{\Xcal_k|\Xcal_k = \{\x[i]\in
\reals^3|i=1,...M_k\}\}_{k=1}^K$ be the \lidar returns of a target.} The
equation~\eqref{eq:OptSolution}becomes 
\begin{equation}
\label{eq:KernelizedOptSolution}
    \min_{\H \in \Sim(3)}
    \sum_{t=1}^T\sum_{k=1}^{K_t}\sum_{i=1}^{M_k}
   k(\x[i], \mu_k) \lvert \n[t][\transpose] \left( \H \cdot \x[i] - \p[0,t] \right)
   \rvert,
\end{equation}
where $\mu_k$ is the centroid of the $k$-th collection on the $t$-th target and
$k(\cdot, \cdot)$ is the squared exponential kernel~\cite[Chapter
4]{rasmussen2006gaussian}
\begin{equation}
\label{eq:Kernel}
k(\x[i], \mu_t) = \sigma\text{exp}\left(\frac{-(\x[i]-\mu_t)^2}{2l}\right),
\end{equation}
where $\sigma$ and $l$ are the signal variance the length scale, respectively.
\end{problem}
}
}

\begin{remark} The transformation $H$ applies to the entire collection of points $\Xcal$. How to define the collection is addressed in Sec.~\ref{sec:CollectionOfPoints}.
\end{remark}

\begin{remark}
In practice each target's normal vector and a point on the target must be estimated
from data; see \cite{huang2019lidartag,huang2020improvements}.
\end{remark}

\subsection{Parsing Points on a Target to Collections of Points}
\label{sec:CollectionOfPoints}
The specific method for parsing a target’s point cloud and the number of calibration
transformations depends on the nature of the sensor. For a \slidarN, the points $\pc$
on a target are typically separated by beam number so each beam has its own
optimization problem~\eqref{eq:OptSolution}. For example, for a $K$-beam \slidarN, a natural choice for
the set of calibration parameters is $\{\H[k] | k=1,\dots,K~\text{and}~\H[k] \in
\Sim(3)\}$, where $H_k$ is the similarity transform for the $k$-th ring. Similarly, one could also split each ring into quarter circles ($90^\circ$ arcs), which would lead to $4K$ calibration parameters. On the one hand, for an OPA-based \sslidarN, we propose to form an $m\times n$ grid over the planar target and
then parse the point cloud based on each point's projection onto the target along the
normal of each grid. This way, the set of calibration parameters becomes
$\{\H[k] | k=1,\dots,m\times n~\text{and}~\H[k] \in \Sim(3)\}$. On the other hand, one may wish to stick with parsing via ``rows or columns''. In the end, it's a choice that has to be made by the calibrator.

\squeezeup
\subsection{Overfitting, Underfitting, and Model Mismatch}
\label{sec:ModelDicussion}

We first point out that the proposed calibration model can suffer from overfitting (i.e., a
unique $\H[][*]$ does not exist) if less than four targets are used to calibrate a
\lidarN. In particular, for a single planar target, the rotation about the target
normal $\n[t]$, two components of the translation vector $\v$ (translation in the
plane of the target), and the scale $s$ are unconstrained. For two planar targets,
translation along the line orthogonal to the targets' normal vectors and scale are
unconstrained. Therefore, understanding how target positioning constrains the
calibration model is essential. As mentioned above, underfitting could occur if the
data is parsed into rings, but say one of the rings has distance errors that vary with rotation angle. In this case, splitting each ring into different arcs will eliminate the underfitting.

Our calibration model is not a phenomenological model. The existing published models postulate how a LiDAR works and hence are phenomenological models. These postulated models result in different numbers of parameters, ranging from three to eight. Our ``calibration model'' abstracts away from the physical device and focuses on transforming the point cloud so that the $(x,y,z)$ values are correct.

\subsection{Target Placement Guideline}
 \label{sec:TargetPlacementGuideline}
Theorem~\ref{thm:uniqueness} given below provides
practically realizable conditions under which a unique answer exists to
Problem~\ref{prob:intrinsic_calibration}. The proof is given in
Appendix~A in~\cite{huang2020global}. We further propose a method to globally compute the
unique solution, see Sec.~\ref{sec:ConvexRelaxation}.


\comment{
\addd{Theorem~\ref{thm:uniqueness} states that there exists one unique answer to
both Problem~\ref{prob:intrinsic_calibration} and
Problem~\ref{prob:KernelizedIntrinsicCalibration}. The proof is given in
Appendix~\ref{sec:Uniqueness}. We further introduce the \LanDual and formulate the
problem as a QCQP. The resulting problem becomes a convex SDP; therefore, 
the global optimum can be found, see Sec.~\ref{sec:ConvexRelaxation}.}
}

\begin{figure}[t]%
\centering
\includegraphics[width=0.9\columnwidth, trim={0.0cm 00cm 0cm 0cm},clip]{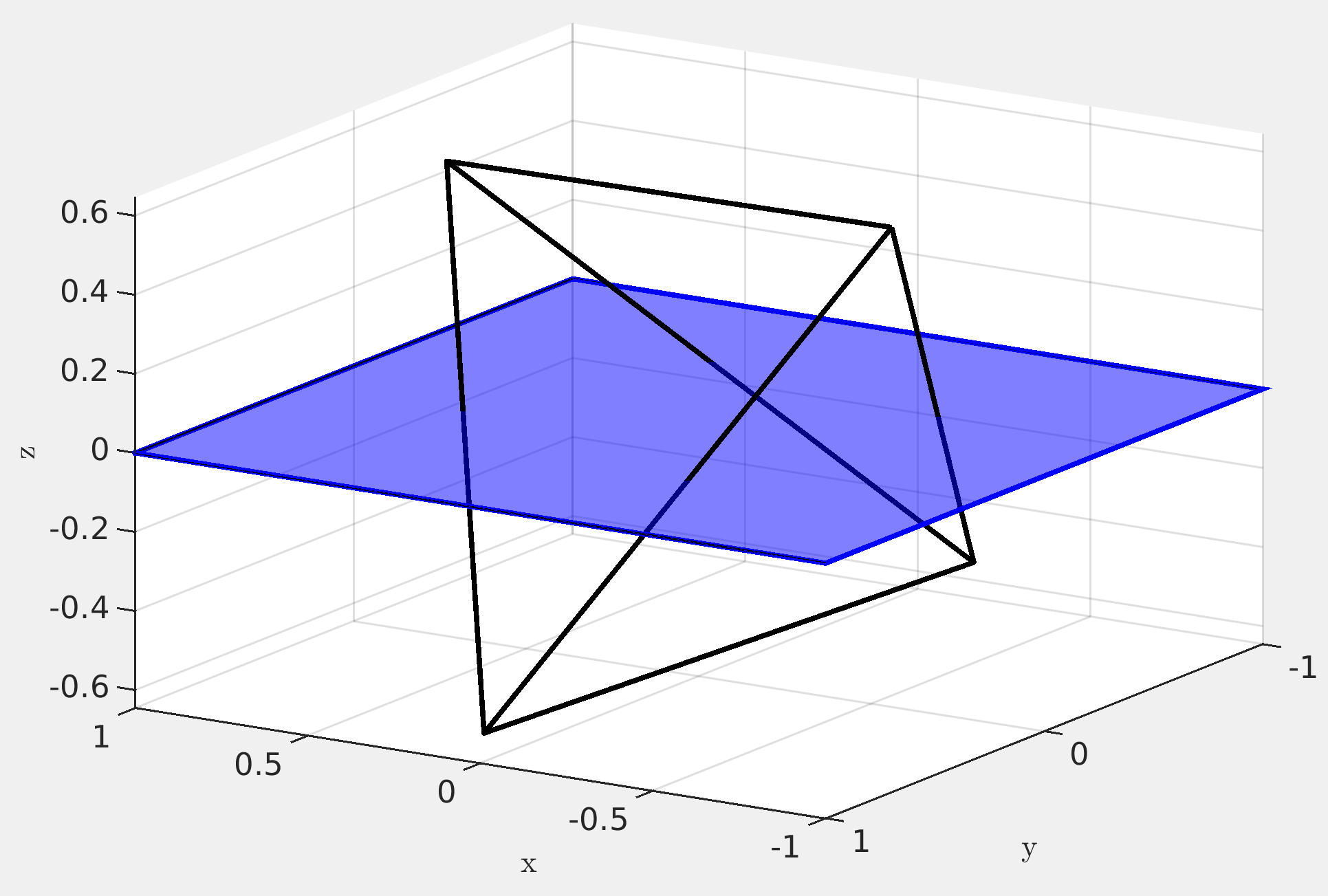}
\caption{An oriented tetrahedron (black) and a ring plane (blue) for target positioning that satisfies the conditions of
Theorem~\ref{thm:uniqueness}.}
\label{fig:oriented-tetrahedron}%
\end{figure}
%

For a subset $\Scal \subset \reals^3$, let $[\Scal] \text{ and } [\Scal]^\perp$ denote its span and
its orthogonal complement, respectively. Let $\{\e[1], \e[2], \e[3]\}$ be the canonical basis for
$\reals^3$. We also denote $\P\subset\reals^3$ a set traced out by a single ring (i.e., a ring plane) of
a perfectly calibrated \slidarN. Without loss of generality, we assume
$\P=[\e[3]]^\perp$.

Consider four targets with unit normal vectors $\n[i]$ and let $\p[0,i]$ be a point
on the $i$-th target. For $1\le i \le 4$, the plane defined by the $i$-th target is
\begin{align*}
\V[i] &:= \p[0,i] + [\n[i]]^\perp\\
&:=  \p[0,i] + \{ v \in \real^3~|~ \n[i][\top] v =0  \},
\end{align*}
the plane formed by the set of vectors orthogonal to the normal vector $ \n[i]$, and subsequently translated by $\p[0,i]$.

%
\begin{assumptionN}[Normal Vectors]
    \label{asm:assumtionN}
    All sets of three distinct vectors from $\{ \n[1], \n[2], \n[3], \n[4], \e[3]\}$
    are \li.
\end{assumptionN}

Under Assumption~\ref{asm:assumtionN}, for each $i \neq j \in \{1, 2, 3, 4 \}$, there exist unique points $\p[ij]:=\P \cap \V[i] \cap \V[j]$. All total, there are $\binom{4}{2} = 6$ such intersection points.

\begin{assumptionB}[Basis Vectors]
    \label{asm:assumtionB}
     Each of the following sets of vectors\footnote{There are $\binom{6}{2}=15$ possible pairs of vectors, of which we use 13 to establish uniqueness.} is a basis of the ring plane: 
    \begin{enumerate}[(a)]
        \item $\{ \p[12], \p[13] \}$, $\{ \p[13], \p[14] \}$, $\{ \p[14], \p[12] \}$,
        \item $\{ \p[12], \p[23] \}$, $\{ \p[23], \p[24] \}$, $\{ \p[24], \p[12] \}$,
        \item $\{ \p[13], \p[23] \}$, $\{ \p[23], \p[34] \}$, $\{ \p[34], \p[13] \}$,
        \item $\{ \p[14], \p[24] \}$, $\{ \p[24], \p[34] \}$, $\{ \p[34], \p[14] \}$,
        \item $\{ \p[14], \p[23] \}$
    \end{enumerate}
\end{assumptionB}

   

We note that for a perfectly calibrated \lidarN, if $\H$ is
         the identity element of $\SE(3)$, then for all $i
    \neq j \in \{1, 2, 3, 4 \}$, it must be true that $\H\cdot\p[ij] \in  \V[i] \cap
     \V[j].$ The following shows that the converse is true when the targets are appropriately positioned and oriented.



\begin{theorem}
    \label{thm:uniqueness}
    Assume that Assumptions~\ref{asm:assumtionN} and~\ref{asm:assumtionB}
     hold and let $\H \in \Sim(3)$. If for each $i
    \neq j \in \{1, 2, 3, 4 \}$, $\H\cdot\p[ij] \in  \V[i] \cap
     \V[j]$, then
     \begin{equation}
         \H =
         \begin{bmatrix}
            \I & \zeros \\
            \zeros & 1
         \end{bmatrix}.
     \end{equation}
 \end{theorem}

\textbf{Proof:} See Appendix~\ref{sec:Uniqueness} in~\cite{huang2020global}.\\
An immediate Corollary of Theorem~\ref{thm:uniqueness} is the uniqueness of the transformation, as stated below.\\
\textbf{Corollary}. Let $\widetilde{\H} \in \Sim(3)$ and define $\widetilde{\p}_{ij}:=\widetilde{\H}(\p[ij])$. Then under Assumptions~\ref{asm:assumtionN} and~\ref{asm:assumtionB}, the only element $\G \in \Sim(3) $ satisfying $\G(\widetilde{\p}_{ij}) \in \V[ij]$ all $i \neq j$ is $\G = \widetilde{H}^{-1}$.\\

Assumption~\ref{asm:assumtionN} assumes linear independence
($\binom{5}{3}=10$) for any three of the target normal vectors, but does not directly address their ``degree of independence,'' that is, their condition number.
Similarly, Assumption~\ref{asm:assumtionB} only involves linear independence. In practice, one wonders if it is important to position the targets so that the two sets of vectors are well conditioned? We will later show that the normal vectors are most important. In Fig.~\ref{fig:oriented-tetrahedron}, we show an oriented tetrahedron\footnote{Each of the four planar targets defines an
infinite plane. The infinite planes create a tetrahedron.} where the
normal vector of each face\footnote{Faces of the tetrahedron are extended planes
of the targets. Therefore, normal vectors of faces are normal vectors of targets.}
satisfies Assumption~\ref{asm:assumtionN} and the orientation of the
tetrahedron intersecting with the ring plane fulfills
Assumption~\ref{asm:assumtionB}.

\begin{remark}
    The plane defined by the $i$-th target is $\mathcal{V}_i:= \p[0,i] +
    \n[i][\perp]$. If we were to only use three targets, there would exist a unique point $\q[0] \in \mathbb{R}^3$ defined by
    \begin{equation}
    \label{eq:PoleOfCalibration}
        \q[0] :=\bigcap\limits_{i=1}^3 \mathcal{V}_i.
    \end{equation} 
    When \eqref{eq:PoleOfCalibration} holds, the scale factor, $s$, is not unique;
    see Fig.~\ref{fig:PropositionIllustration}. 
    In other words, a ring plane passing through the point $\q[0]$ leads to the scale factor $s$ being unconstrained and therefore a unique answer does not exist. Therefore, it was critical to analyze target placement for the proposed method. A similar analysis seems to be missing for the baseline methods.
\end{remark}

\begin{figure}[t]%
\centering
\includegraphics[width=0.9\columnwidth, trim={0cm 0cm 0cm 0cm},clip]{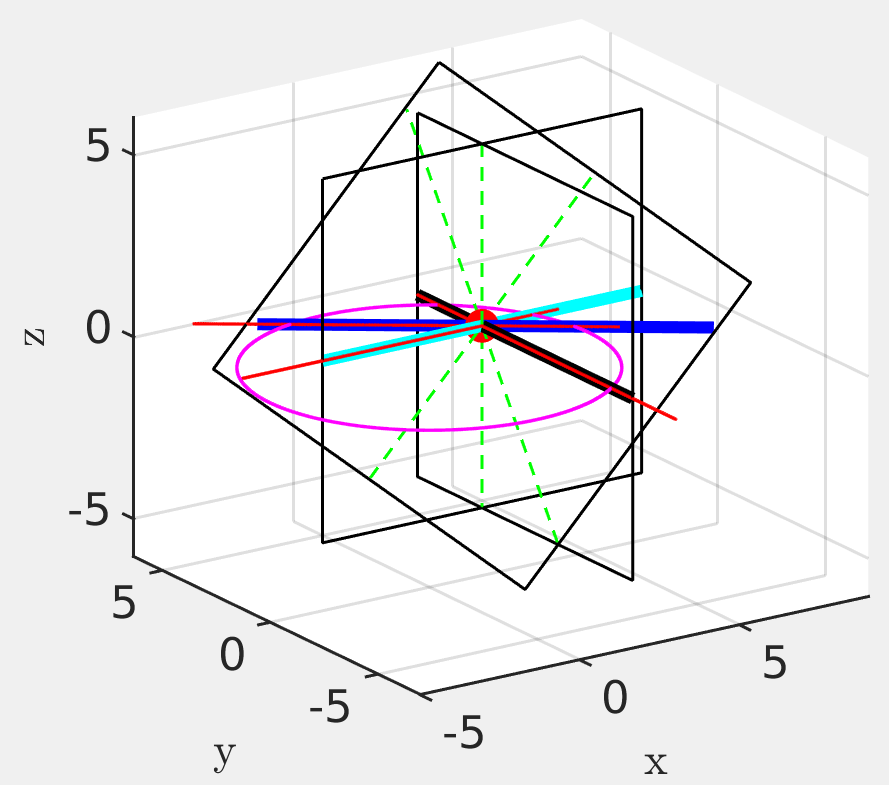}
\caption[]{Illustration of a degenerate case. The black squares are 
different planes with the normal vectors being \lind. The dotted green lines indicate the intersection of two planes. The magenta circle shows the
\lidar ring passing through $q_0$ (red center bullet) in \eqref{eq:PoleOfCalibration}. The thick blue, cyan
and black lines are the intersections of the ring planes and the targets. The P2P cost of
the intersections are therefore zero. The red lines show the results of 
applying a transformation $\H = \begin{bmatrix}s\I & (1-s)\q[0]\\\zeros&
1\end{bmatrix}$ to the intersection lines. The cost remains zero because the transformed lines
stay on the targets. 
}%
\label{fig:PropositionIllustration}%
\squeezeup
\end{figure}

\section{Baseline Intrinsic Calibration Methods}
\label{sec:Baseline}
There are two standard calibration models for spinning \lidarsN. They are
based on spherical coordinates $(\rho, \theta, \phi)$, referred to as range,
elevation and azimuth, 
\begin{equation}
\label{eq:spherical2Cartesian}
\begin{bmatrix}
x\\y\\z
\end{bmatrix} = f(\rho, \theta, \phi)
=\begin{bmatrix}
\rho\cos\theta\sin\phi\\
\rho\cos\theta\cos\phi\\
\rho\sin\theta
\end{bmatrix},
\end{equation}
and the inverse function, 
\begin{equation}
\label{eq:Cartesian2spherical2}
\begin{bmatrix}
\rho\\\theta\\\phi
\end{bmatrix} = f^{-1}(x,y,z)
=\begin{bmatrix}
    \sqrt{x^2+y^2+z^2}\\
    \asin[\frac{z}{\sqrt{x^2+y^2+z^2}}]\\
    {\rm atan2}(x,y)
\end{bmatrix}.
\end{equation}


\subsection{3-Parameter Model}
The most basic model~\cite{pandey2010extrinsic} assumes the \lidar measurements are
made in spherical coordinates, $(\rho, \theta,\phi)$. Corrections to a measurement
are given by a collection of offsets \mbox{$\alpha:=(\delta_\rho, \delta_\theta,
\delta_\phi).$}  Expressing the calibrated measurement in Cartesian coordinates gives
\begin{equation}
\label{eq:SimpleCalibration}
        \Gamma_{\alpha}(\rho, \theta, \phi) := 
        \begin{bmatrix}
    (\rho+\drho)\cos(\dtheta)\sin(\phi-\dphi)\\
    (\rho+\drho)\cos(\dtheta)\cos(\phi-\dphi) \\
    (\rho+\drho)\sin(\dtheta)
        \end{bmatrix}
\end{equation}

\begin{remark} 
    \textbf{(a)} The nominal elevation for each ring is taken as zero;,
i.e., $\theta=0.$  \textbf{(b)} While it is not necessarily a drawback, most \lidar
 interfaces only return a Cartesian representation of a measured point. For example, the ROS driver from Velodyne only provides the Cartesian coordinate and the intensity of a point. Hence, the
user must assure the transformation to spherical coordinates, make the measurement
correction coming from the calibration, and then transform back to Cartesian
coordinates. 
\end{remark}

\squeezeup
\subsection{6-Parameter Model}
This model \cite{glennie2010static, nouira2016point} also works in spherical coordinates. In addition to the three offsets
above, it includes $s$, a scale factor of $\rho$, $h$, a horizontal offset of the
origin, and $v$, a vertical offset for the origin. The correction model becomes
\begin{equation}
\label{eq:ComplexCalibration}
         \bar{\Gamma}_{\alpha}:= 
\left[
\begin{array}{l}
    (s\rho+\delta_\rho)\cos\delta_\theta\sin(\phi-\delta_\phi) - h\cos(\phi-\delta_\phi)\\
    (s\rho+\delta_\rho)\cos\delta_\theta\cos(\phi-\delta_\phi) + h\sin(\phi-\delta_\phi)\\
    (s\rho+\delta_\rho)\sin\delta_\theta + v\\
\end{array}\right],
\end{equation}
and therefore $\alpha:=(\delta_\rho, \delta_\theta,
\delta_\phi, s, h, v).$ 

\subsection{Can these transformations be expressed by elements of $\Sim(3)$?}
The short answer is no. When viewed in Cartesian coordinates, the calibration model
\eqref{eq:ComplexCalibration} is fundamentally nonlinear and can be expressed as $\R[1](\R[2]\t[1] + \t[2])$,
where $\R[1], \R[2], \t[1] \text{ and } \t[2]$ depend on the measured point, 
$$\x[][\transpose] = [x,y,z] \xrightarrow[]{f^{-1}} (\rho, \theta, \phi),$$
\begin{align}
    \label{eq:ExpandComplexModelMatrixForm}
    \R[1]&=
             \begin{bmatrix}
                 \sin(\phi-\dphi) & -\cos(\phi-\dphi) & 0\\
                 \cos(\phi-\dphi) &  \sin(\phi-\dphi) & 0\\
                 0 &0 &1
             \end{bmatrix},\\ \nonumber
    \R[2]\t[1]+\t[2] &=
             \begin{bmatrix}
                 \cos(\dtheta) & 0 & -\sin(\dtheta)\\
                 0 &1 &0\\
                 \sin(\dtheta) & 0 & \cos(\dtheta)
             \end{bmatrix}
             \begin{bmatrix}
                 s\rho+\drho\\
                 0
                 \\0
             \end{bmatrix} +
             \begin{bmatrix}
                 0\\h\\
                 v
             \end{bmatrix}.
\end{align}
The calibration first applies a corrected elevation angle to a range value that has been offset and scaled. Next, the  $(y,z)$ position of the origin is offset, and lastly, a corrected rotation about the $z$-axis is applied.

\squeezeup
\subsection{Cost function for baseline models}
\label{sec:ProblemStatement}

 For a given collection of points ${\cal PC}$, possibly from multiple targets
 $t\in\{1, \ldots, T\}$, we seek calibration parameters that solve 
\begin{equation}
\label{eq:OptSolutionBaseline}
    \min_{\alpha} 
   \sum_{t=1}^T\sum_{i=1}^{M_t} 
    | \n[t][\transpose] \left(F(\x[i], \alpha) - \p[0,t] \right) |,
\end{equation}
where 
\begin{enumerate}
       \item for baseline model $\BL_1$~\cite{pandey2010extrinsic} in \eqref{eq:SimpleCalibration}, $\alpha =(\drho, \dtheta, \dphi)$ and
     $$F({\x}, \alpha):= \Gamma_{\alpha} \circ f^{-1}({\x});$$
      \item for baseline model $\BL_2$~\cite{glennie2010static, nouira2016point} in \eqref{eq:ComplexCalibration}, $\alpha =(\drho, \dtheta, \dphi, s, h, v)$ and 
      $$F({\x}, \alpha):= \bar{\Gamma}_{\alpha} \circ f^{-1}({\x}).$$
\end{enumerate}

\begin{remark}
   For the 3-parameter model $\BL_1$, a single planar target is sufficient to
   uniquely determine a calibration. For the 6-parameter model $\BL_2$,
   non-uniqueness can occur as outlined in Theorem~\ref{thm:uniqueness} and
   Fig.~\ref{fig:PropositionIllustration}.
\end{remark}

\squeezeup
\section{Global Optimization: from Min to Min-Min}
\label{sec:ConvexRelaxation}
Theorem~\ref{thm:uniqueness} states that a unique answer exists. In this section, we
propose a method to find it. In particular, we show a
globally convergent algorithm to determine the element of the Lie Group that globally
minimizes the cost function~\eqref{eq:P2PDistance}. This is done by converting the minimization
problem in~\eqref{eq:OptSolution} to a Min-Min problem. The inner minimization problem has an efficient global solution by results in~\cite{olsson2008solving, briales2017convex}. The outer minimization is scalar, and it is easy to limit the range of interest to a compact set.

\squeezeup
\subsection{Problem Statement and Compact Sets}
\label{sec:CompactSets}
The intrinsic calibration problem is defined in~\eqref{eq:OptSolution}:
\begin{align}
    \label{eq:minOverSim}
    \min_{s,\R,\v}J(s,\R,\v; \Xcal_t, \n[t],\p[0,t]),
\end{align}
where the cost is defined as
\begin{equation}
    \label{eq:P2PDistance2}
  J(s, \R, \v):= \sum_{t=1}^T\sum_{i=1}^{M_t}
    \lvert \n[t][\transpose] \left( s\R \cdot \x[i] +\v- \p[0,t] \right) \rvert^2.
\end{equation}
It is straightforward to show that $J$ is a continuous function on $\Sim(3)$. 
As a set, we write $\Sim(3) = \reals^+ \times \SO(3) \times \reals^3$, where $\reals^+:= \reals > 0$. Note that $\SO(3)$ is compact. To guarantee that a minimum exists, we consider compact subsets of $s\in\reals^+$ and $\v\in\reals^3$ and define
\eqref{eq:minOverSim}:
\begin{equation}
\begin{aligned}  
    J^\ast &:= \min_{s,\R,\v} J(s,\R,\v)\\
    s^\ast, \R[][\ast], \v[][\ast] &:= \argmin_{s,\R,\v} J(s,\R,\v).
\end{aligned}
\label{eq:GoldMine}
\end{equation}

\begin{algorithm}[t]
\DontPrintSemicolon
\KwInput{Collections of points ($\Xcal_t$) and corresponding targets ($\n[t], \p[0,t]$)}
\KwOutput{Global optimum $\H[][*]\in \Sim(3)$}
\KwInitialization{Arbitrary initialization is allowable. We use $\R = \I, \v = \zeros,$ and $s \in [0.8, 1.2]$}
  \While{$k<$ MAX\_ITER}
  {
  \tcp{Scale the collections of points}
  $\pre[k][U]\Xcal_t^s \leftarrow \pre[k][U]s\cdot\Xcal_t$,~~
  $\pre[k][L]\Xcal_t^s \leftarrow \pre[k][L]s\cdot\Xcal_t $\;

  \tcp{Compute the optimum $\R,\v$ given the scaled points}
  \scalebox{0.9}{$\pre[k][L]\R\leftarrow \pre[k][L]\Xcal_t^s\text{~and~(42) in \cite{huang2020global}},~\pre[k][L]\v \leftarrow \pre[k][L]\R\text{~and~(36) in \cite{huang2020global}}$}\;
  \scalebox{0.9}{$\pre[k][U]\R\leftarrow \pre[k][U]\Xcal_t^s\text{~and~(42) in \cite{huang2020global}} ,~\pre[k][U]\v \leftarrow \pre[k][U]\R\text{~and~(36) in \cite{huang2020global}}$}\;

  \tcp{Compute the P2P by current $\R,\v$}
  $\pre[k][L]J\leftarrow$ (13) in \cite{huang2020global},~~
  $\pre[k][U]J\leftarrow$ (13) in \cite{huang2020global}\;
  \tcp{Update $\pre[k+1][U]s$ and $\pre[k+1][L]$}
  $\pre[k+1][L]s\leftarrow$ (21) in \cite{huang2020global},~~
  $\pre[k+1][U]s\leftarrow$ (21) in \cite{huang2020global}\;
  }
  \Return $s^*, \R[][*],\v[][*]$

  %
\caption{Proposed global optimizer}
\label{alg:ConvexAlg}
\end{algorithm}

\squeezeup
\squeezeup
\subsection{From Min to Min-Min}
\label{sec:MinMin}
We use again the fact that $\Sim(3) = \reals^+ \times \SO(3) \times \reals^3$ to redefine \eqref{eq:minOverSim}
as a \mm problem.
\begin{proposition}
\label{prop:MinToMinmin}
\begin{equation}
    \label{eq:MinEqualMinmin}
    \min_{s,\R,\v} J(s,\R,\v) = \min_s\min_{\R,\v}J(s,\R,\v),
\end{equation}
where the minimums are taken over the same compact sets used in \eqref{eq:GoldMine}.
\end{proposition}
\textbf{Proof:} 
From \eqref{eq:GoldMine}, we have that
\begin{equation}
    J^\ast = J(s^\ast, \R[][\ast], \v[][\ast]).
\end{equation}
We further introduce
\begin{equation}
\begin{aligned}
h(s) &:= \min_{\R,\v}J(s,\R,\v)\\
h^\ast&:=\min_s h(s) 
\end{aligned}
\end{equation}
and note that by definition, we have
\begin{equation}
    \label{eq:JF}
    J^\ast \le h^\ast.
\end{equation}
To show $h^\ast \le J^\ast$, we use that minimizing values of $J$ exist and hence
\begin{equation}
    \label{eq:FJ}
    h^\ast \le h(s^\ast) := {\min_{\R,\v}J(s^\ast, \R,\v)} \le J(s^\ast, \R[][\ast], \v[][\ast]) = J^\ast.
\end{equation}
Equations \eqref{eq:JF} and \eqref{eq:FJ} conclude the equivalency of \eqref{eq:MinEqualMinmin}.\\
\qed

From Proposition~\ref{prop:MinToMinmin}, we can define the intrinsic calibration
problem as
\begin{align}
    \label{eq:minS}
    f(s) &:= \min_{\R, \v}J(s,\R,\v;\Xcal_t, \n[t],\p[0,t])\\
    J^\ast &= \min_{s}f(s).
\end{align}

\comment{
\begin{itemize}
    \item $J^\ast=\min_{s,\R,\v}J(s,\R,\v) \implies \exists s^\ast, R^\ast, t^\ast$ such that $$ J^\ast = J(s^\ast, R^\ast, t^\ast)$$ 
    \item $f(s) := {\rm inf}_{\R,\v}J(s, \R,\v)$
    \item $f^\ast={\rm inf}_{s}f(s) \implies f^\ast \le f(s)$ for all $s$
    \item By definition, $J^\ast \le f^\ast $
\item  To show: $f^\ast \le J^\ast$. Consider 
$$f^\ast \le f(s^\ast) := {\inf_{\R,\v}J(s^\ast, \R,\v)} \le J(s^\ast, R^\ast, t^\ast) = J^\ast$$
\end{itemize}

Due to the property of the Cartesian product, given a scaling
$s\in\reals^+$, \eqref{eq:minOverSim} can be written as
\begin{align}
    \label{eq:minS}
    f(s) = \min_{\R,\v}J(\R,\v;s,\Xcal_t, \n[t],\p[0,t]),
\end{align}
which has an efficient global solution~\cite{olsson2008solving, briales2017convex}.
Therefore, \eqref{eq:minOverSim} is a \mm problem:
\begin{equation}
    \label{eq:minmin}
    \min_{s}\min_{\R,\v}\sum_{t=1}^T J_t(s,\R,\v; \Xcal_t, \n[t],\p[0,t]).
\end{equation}
\begin{itemize}
    \item $J^\ast=\min_{s,\R,\v}J(s,\R,\v) \implies \exists s^\ast, R^\ast, t^\ast$ such that $$ J^\ast = J(s^\ast, R^\ast, t^\ast)$$ 
    \item $f(s) := {\rm inf}_{\R,\v}J(s, \R,\v)$
    \item $f^\ast={\rm inf}_{s}f(s) \implies f^\ast \le f(s)$ for all $s$
    \item By definition, $J^\ast \le f^\ast $
\item  To show: $f^\ast \le J^\ast$. Consider 
$$f^\ast \le f(s^\ast) := {\inf_{\R,\v}J(s^\ast, \R,\v)} \le J(s^\ast, R^\ast, t^\ast) = J^\ast$$
\end{itemize}
}

\squeezeup
\subsection{Determining the Calibration Parameters}
Because we can bound the scaling $s$ to a compact set, a minimizing value can be found by dense search. In the Appendix~C\cite{huang2020global}, we discuss a heuristic to speed up the algorithm if one feels the need.  


At the $(k+1)$-th iteration, the scaling parameter is determined and the remaining
parameters correspond to $\SE(3)$. To solve for them, we adopt techniques that
were used to solve 3D registration or 3D SLAM~\cite{tron2015inclusion,
carlone2015lagrangian, olsson2008solving, briales2017convex} where the problem is
formulated as a QCQP, and the Lagrangian dual relaxation is used. The relaxed
problem becomes a Semidefinite Program (SDP) and convex. The problem
can thus be solved globally and efficiently by off-the-shelf specialized
solvers~\cite{grant2014cvx}. This process is summarized in
Appendix~B in~\cite{huang2020global}. As shown in~\cite{briales2017convex}, the dual
relaxation is empirically always tight (the duality gap is zero). We also use the
same simulation and experimental data set in~\cite{olsson2008solving,
briales2017convex} to verify that the proposed algorithm is able to converge
globally, see Appendix~C in~\cite{huang2020global} for more visual results.

\begin{remark}
One may argue that the above is a standard optimization problem in that other solvers, such as Google ceres or g2o, can easily solve it. However, they are local solvers and will suffer if a good initial guess is not provided.
\end{remark}

\section{Simulation and Experimental Results}
\label{sec:SimulationResults}

This section first presents a simulation study of the intrinsic calibration problem
that will show the importance of Theorem~\ref{thm:uniqueness} for eliminating over
parameterization of the proposed model. Results here will then be used to inform the
experimental work. All experiments are conducted with a \velodyne, mounted on an
in-house designed torso for a Cassie-series bipedal
robot~\cite{CassieAutonomy2019ExtendedEdition}. All the optimizations related to the
proposed methods are solved via the proposed algorithm (see
Algorithm~\ref{alg:ConvexAlg}). As for the two baselines, we implemented them and
solved via the \texttt{fmincon} function in MATLAB. All the methods, simulator,
solver, and datasets are open-sourced:~\cite{githubFileIntrinsic,
githubFileLiDARSimulator, githubFileGlobalSim3Solver}. 

\begin{table*}[b]
    \caption{Validation data for various calibration methods on a \velodyneN. The numbers are the average P2P distance in meters. For each column, the \lidarN's calibration parameters are optimized on a common set of four targets arranged as a tetrahedron and validated on a common set of three targets arranged in different orientations and at different distances. What varies in each column is the level of systematic noise added to the raw data, as reflected in the P2P error reported in the row for Factory Calibration. Noise 0 is the raw signal from the factory-calibrated \lidarN, while Noise 1 to Noise 7 have increasing levels of systematic error. Our proposed method is insensitive to systematic error. }
\label{tab:Exp}
\center
\footnotesize
\scalebox{1.13}{
    \begin{tabular}{|c|c|c|c|c|c|c|c|c|c|}
\hline
 &
  Methods &
  Noise 0 &
  Noise 1 &
  Noise 2 &
  Noise 3 &
  Noise 4 &
  Noise 5 &
  Noise 6 &
  Noise 7 \\ \hline
\multirow{4}{*}{P2P} &
  Factory Calibration &
  0.0140 &
  0.0196 &
  0.0309 &
  0.0408 &
  0.0486 &
  0.0581 &
  0.0669 &
  0.0772 \\ \cline{2-10}
 &
  Baseline1 (3 parameters) &
  0.0052 &
  0.0054 &
  0.0064 &
  0.0085 &
  0.0111 &
  0.0135 &
  0.0154 &
  0.0164 \\ \cline{2-10}
 &
  Baseline2 (6 parameters) &
  \textbf{0.0042} &
  0.0059 &
  0.0070 &
  0.0100 &
  0.0121 &
  0.0148 &
  0.0165 &
  0.0181 \\ \cline{2-10}
 &
  Our Method &
  0.0047 &
  \textbf{0.0042} &
  \textbf{0.0040} &
  \textbf{0.0041} &
  \textbf{0.0043} &
  \textbf{0.0041} &
  \textbf{0.0039} &
  \textbf{0.0040} \\ \hline
\multirow{4}{*}{\begin{tabular}[c]{@{}c@{}}Point Clouds\\ Thickness\end{tabular}} &
  Factory Calibration &
  \multicolumn{1}{l|}{0.0626} &
  \multicolumn{1}{l|}{0.0777} &
  \multicolumn{1}{l|}{0.1182} &
  \multicolumn{1}{l|}{0.1579} &
  \multicolumn{1}{l|}{0.1981} &
  \multicolumn{1}{l|}{0.2379} &
  \multicolumn{1}{l|}{0.2788} &
  \multicolumn{1}{l|}{0.3211} \\ \cline{2-10}
 &
  Baseline1 (3 parameters) &
  \multicolumn{1}{l|}{\textbf{0.0439}} &
  \multicolumn{1}{l|}{\textbf{0.0463}} &
  \multicolumn{1}{l|}{0.0634} &
  \multicolumn{1}{l|}{0.0817} &
  \multicolumn{1}{l|}{0.0998} &
  \multicolumn{1}{l|}{0.1177} &
  \multicolumn{1}{l|}{0.1347} &
  \multicolumn{1}{l|}{0.1501} \\ \cline{2-10}
 &
  Baseline2 (6 parameters) &
  \multicolumn{1}{l|}{0.0457} &
  \multicolumn{1}{l|}{0.0535} &
  \multicolumn{1}{l|}{0.0719} &
  \multicolumn{1}{l|}{0.1005} &
  \multicolumn{1}{l|}{0.1110} &
  \multicolumn{1}{l|}{0.1308} &
  \multicolumn{1}{l|}{0.1440} &
  \multicolumn{1}{l|}{0.1694} \\ \cline{2-10}
 &
  Our Method &
  \multicolumn{1}{l|}{0.0476} &
  \multicolumn{1}{l|}{0.0475} &
  \multicolumn{1}{l|}{\textbf{0.0476}} &
  \multicolumn{1}{l|}{\textbf{0.0582}} &
  \multicolumn{1}{l|}{\textbf{0.0531}} &
  \multicolumn{1}{l|}{\textbf{0.0464}} &
  \multicolumn{1}{l|}{\textbf{0.0449}} &
  \multicolumn{1}{l|}{\textbf{0.0450}} \\ \hline
\end{tabular}
}
\end{table*}


\subsection{Simulated \lidar Environment}
\label{sec:Simulator}
Due to lack of ground truth, we built a \lidar simulator to compare our
calibration method against the two baselines and to illustrate the role of target
positioning. Our simulator can model \lidar sensors of different
working principles (spinning vs solid-state). With our simulator, it is easy to
control sources of uncertainty, including both mechanical model uncertainty and
measurement noise. Targets are assumed to be planar and polygonal.

To simulate 3D points on a planar target, we generate rays from the \lidar sensor,
and define a ``target'' point as the point at which the ray intersects the target.
The simulator has an option to account for shadowing or not. After locating the exact
\lidar returns on the target, based on the \lidar type, we then add two types of
uncertainty to the returns and report them as measured data. 

\begin{remark} 
        The simulator first finds all intersection points with the (infinite) plane
    defined by the target.  To determine if a point on the plane is within the
    boundary of the target polygon is a well-known problem in computer
    graphics, called the Point-In-Polygon (PIP) problem~\cite{moscato2019provably, el2020enhanced}. The Winding Number algorithm
    \cite{alciatore1995winding} is implemented in this simulator. If the winding
    number of a point is not zero, the point lies inside the boundary; otherwise it
    is outside.
\end{remark}

\begin{figure}[t]%
\centering
\includegraphics[width=1\columnwidth, trim={0cm 0cm 0cm 0cm},clip]{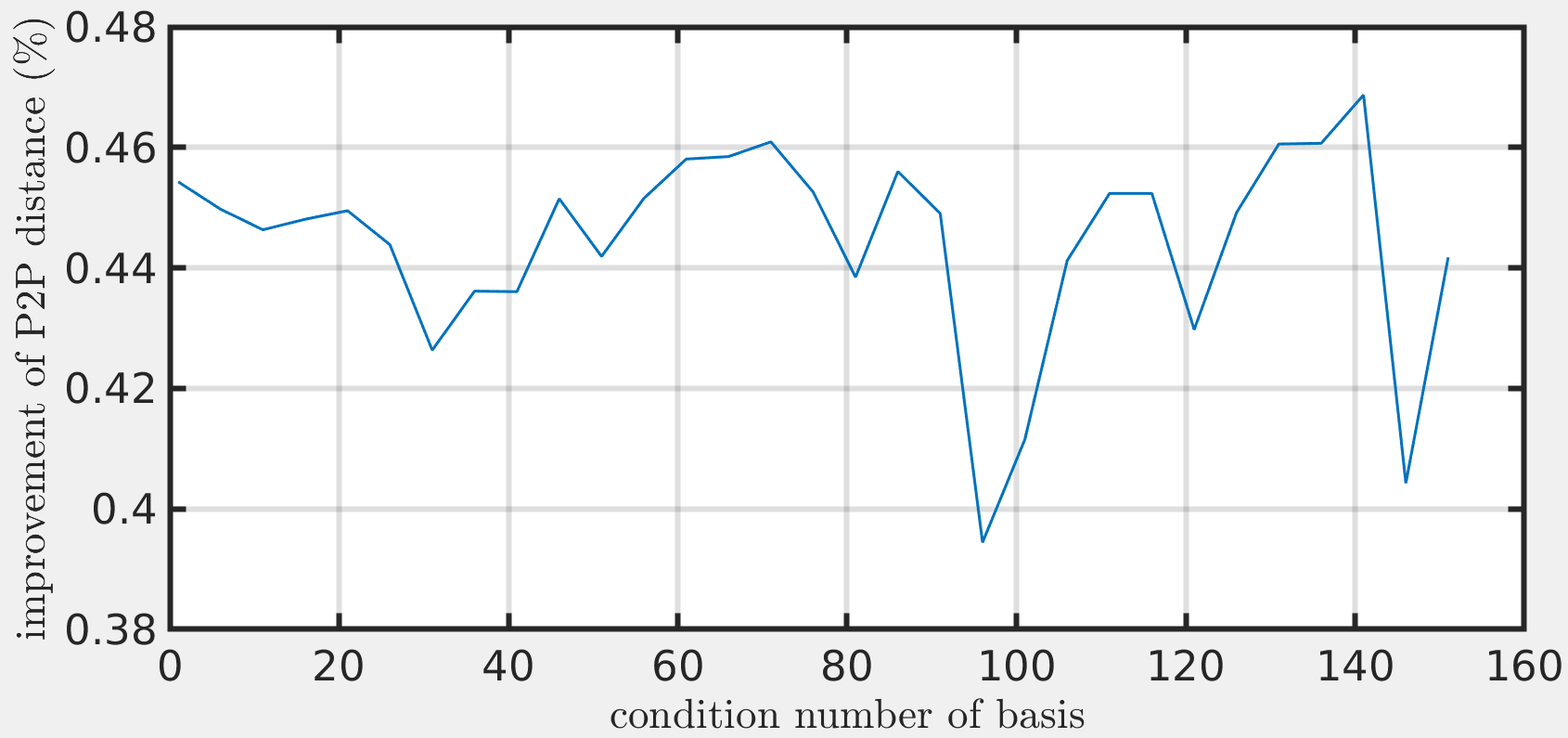}
\caption[]{P2P cost improvement is insensitive to condition number of the basis. The figure illustrates validation results for ten
thousand orientations of four targets arranged as a tetrahedron. Each orientation is a calibration scene from which a set of
calibration parameters is obtained. The calibration parameters are then
applied to another complex scene for validation. This figure shows that regardless of
the orientation of the tetrahedron, we are able to reduce the P2P cost by
44.7\% on average.
} 
\label{fig:OrientationVSP2P}%
\squeezeup
\end{figure}


\squeezeup
\subsection{Intrinsic Calibration of Spinning \lidars}
\label{sec:SimSpingHead}

\subsubsection{Illustration of Theorem~\ref{thm:uniqueness} in simulation} We randomly create 32 deterministic noise sequences (of length appropriate for each ring) and apply them to the ideal measurements of each ring in the \lidar simulator. As a calibration scene, we arrange four targets
around the LiDAR in the shape of a regular tetrahedron and then minimize the cost function in \eqref{eq:OptSolution}. For validation, we collect simulated data from a complex scene with 24 targets at various orientations and distances, using the
same uncalibrated \lidar corrupted by the same deterministic noise. We then applied the calibration parameters to the perturbed uncalibrated \lidar measurements and reduced the P2P distance of the validation scene by 45.43\% compared to the uncalibrated P2P distance.
Additionally, to study the influence of different 
condition numbers in Assumption~\ref{asm:assumtionB} 
on the calibration results, we rotate the tetrahedron with
various orientations (ten thousand different orientations in total, yielding ten thousand different bases).
We observe that the P2P cost improvement is insensitive to the condition number of the basis.
In Fig.~\ref{fig:OrientationVSP2P}, the simulation results show that by using the
tetrahedron (regardless of the orientation) to calibrate the \lidarN, 
the P2P distance was reduced by 44.7\% on average.

\subsubsection{Illustration of Theorem~\ref{thm:uniqueness} in experiments} 
Informed by the simulation study, two scenes (calibration and validation) of data from a \velodyne were collected using the
\lidart package~\cite{githubFileLiDARTag,huang2019lidartag}. In Fig.~\ref{fig:Training}, the calibration scene contains four targets\footnote{The targets were placed a distance of 0.7 meter so that most of the rings would lie on the targets.} arranged as an oriented tetrahedron as stated in Theorem~\ref{thm:uniqueness}. The targets'
normal vectors are estimated using the L1-inspired method
in~\cite{huang2020improvements}. 


In the raw data, the rings are mis-calibrated (by assumption). Because the target planes are estimated using 
data corresponding to mis-calibrated rings, the estimated normal vectors are inaccurate. We therefore propose a refinement of the normal vectors conditioned on the current
estimate of the similarity transformation $\Sim(3)$, that is, based on the current
estimate of the intrinsic parameters $(s, \R,\v)$. Once the normal vectors are updated,
new estimates for $(s, \R, \v)$ can be obtained, resulting in an alternating two-step
process.

Figure~\ref{fig:SystematicError} shows the experimental result of applying the calibration parameters to another set of targets (validation scene) placed at 0.7 meter. The proposed alternating method reduces the P2P distance by 68.6\% with respect to the factory-calibrated \lidarN.
Additionally, we induce
varying levels of systematic errors (from 1 cm to 7 cm, corresponding to level 1 to level 7) to the experiment data to
ensure the robustness of the proposed method,
as shown in Table~\ref{tab:Exp}. We observe that our proposed method is insensitive to systematic error. In addition, Table~\ref{tab:ScansVsPnP} shows the consistency of the proposed method as we accumulate different numbers of scans with different noise levels. It is seen that with higher levels of induced deterministic noise, the performance of the proposed method is invariant to the number of accumulated scans, while with lower levels of deterministic noise, the more scans included, the better is the performance. Both of these properties are desirable.


\begin{remark}
    There are two ways to orient the targets to form a tetrahedron: \textbf{1)} put the targets around the \lidar and keep the \lidar pose fixed; \textbf{2)} keep the target poses fixed and rotate the \lidar. This will allow the targets to form a tetrahedron even if a \lidar does not have a $360^\circ$ field of view.
\end{remark}
\begin{remark}
    To decide if the P2P distance of a ring to a target is random noise or
    systematic errors, we first compute the standard deviation of the \lidar returns
    on the target and consider the P2P distance greater than $3\sigma$ as
    systematic error. 
\end{remark}

\begin{figure}[t]%
\centering
\subfloat{%
\includegraphics[height=0.21\textheight, trim={0cm 0cm 0cm 0cm},clip]{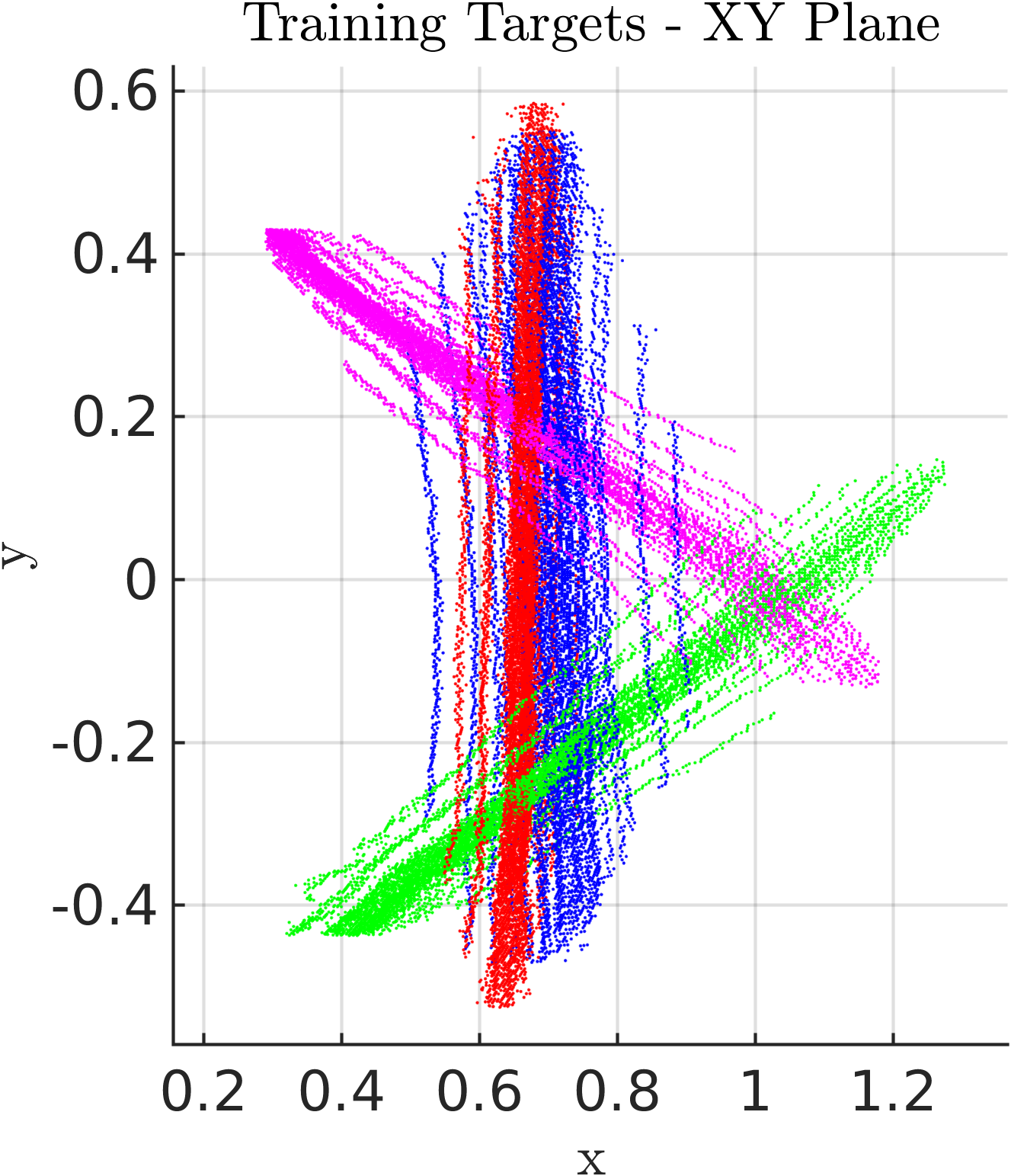}}~
\subfloat{%
\includegraphics[height=0.21\textheight, trim={0cm 0cm 1.5cm 0cm},clip]{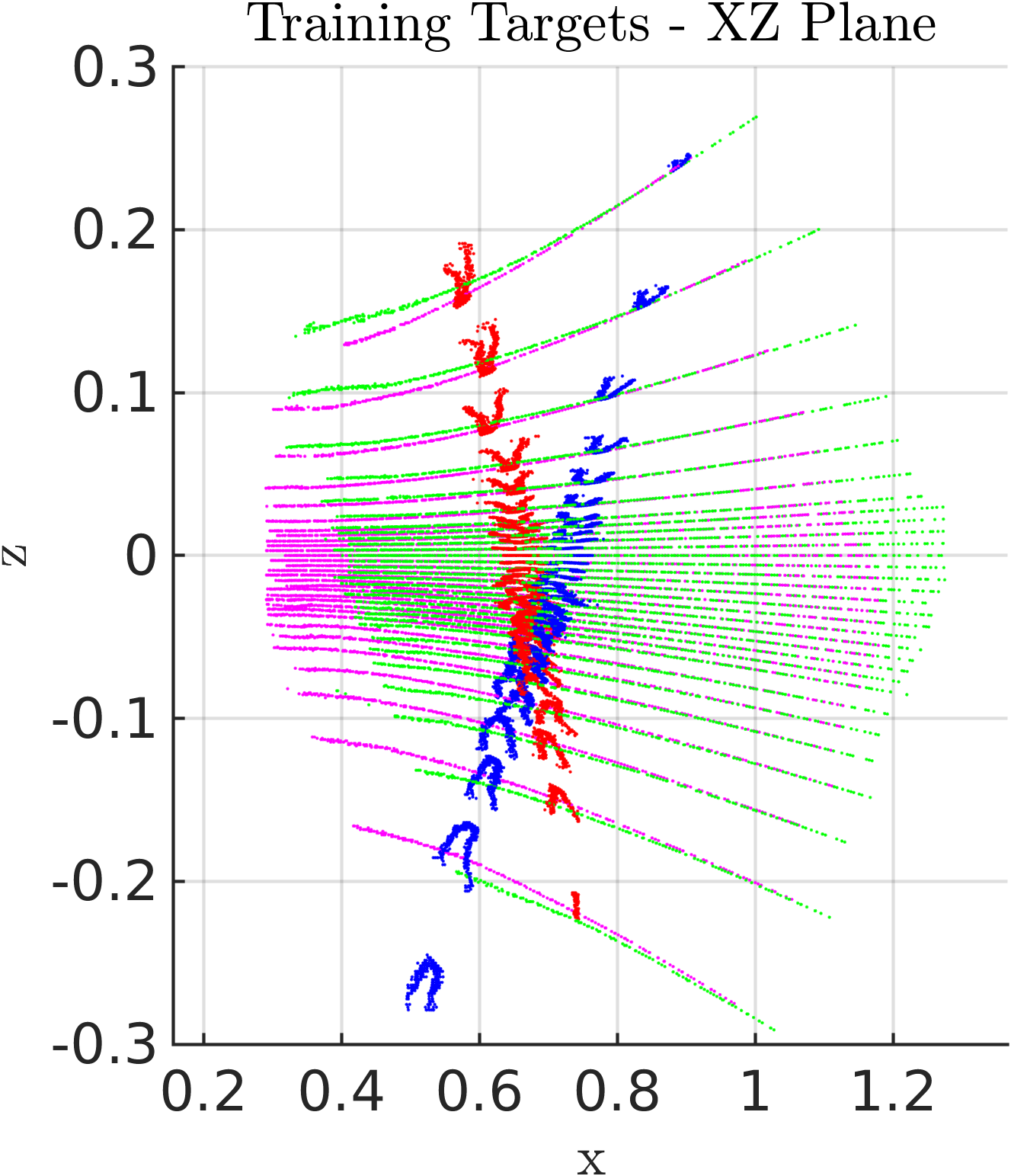}}
\caption[]{Four targets are placed to form a tetrahedron. The left shows the top-down view ($X$-$Y$ plane) of a single scan of the point cloud from the calibration scene and the right shows the side view ($X$-$Z$ plane) of the scene. Different colors stand for different targets.}
\label{fig:Training}%
\squeezeup
\end{figure}

\begin{figure}[t]%
\centering
\subfloat{%
\includegraphics[height=0.8\columnwidth, trim={0cm 0cm 4cm 0cm},clip]{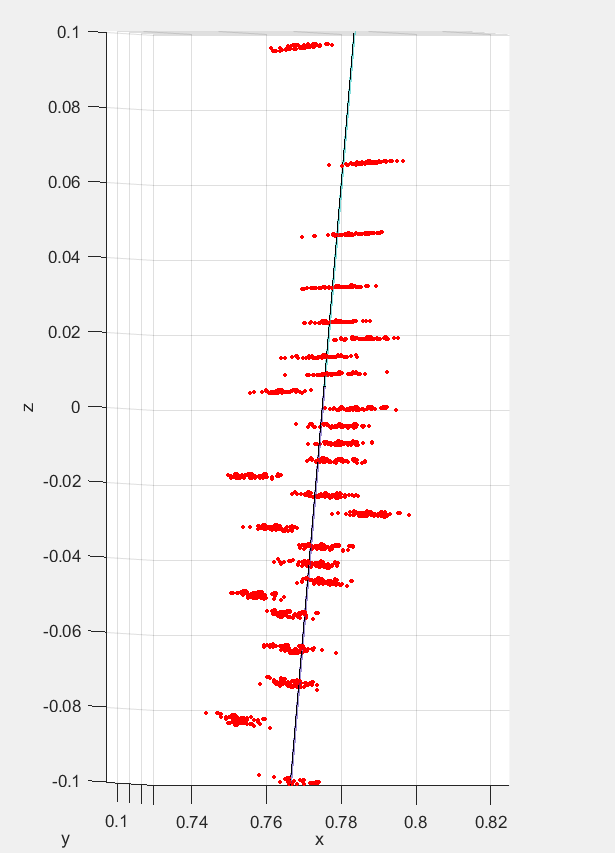}}
~
\subfloat{%
\includegraphics[height=0.8\columnwidth, trim={2cm 0cm 6cm 0cm},clip]{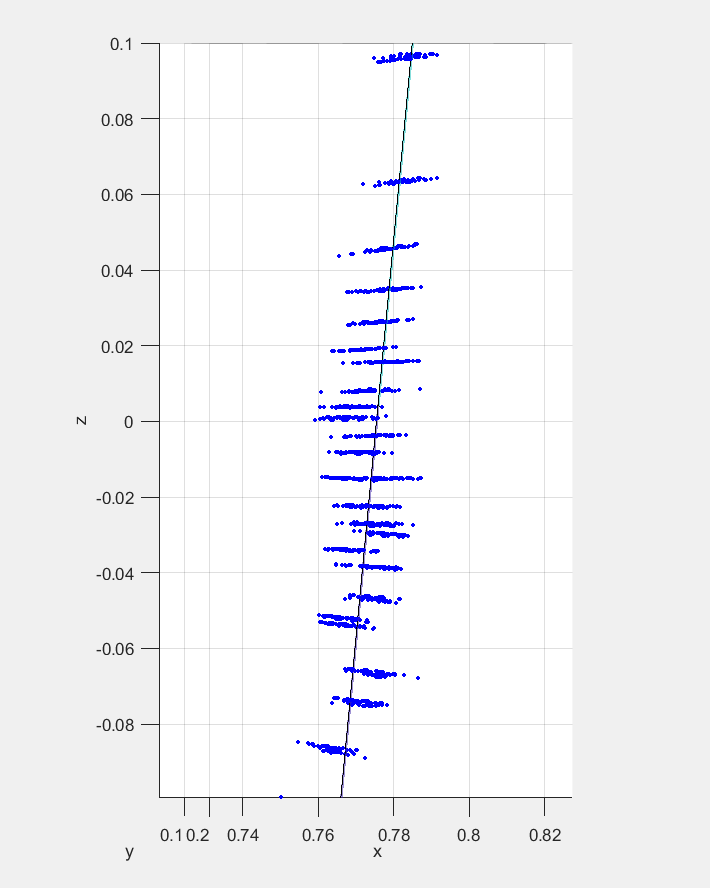}}~
\caption[]{On the left shows a factory-calibrated \velodyne measuring on a planer
target placed at 0.7 meter. The measurements are not consistent and lead to a thickness of 6.2 cm
point cloud with 1.4 cm P2P distance. 
On the right shows the calibration results of the calibration parameters globally optimized over the tetrahedron calibration scene in Fig.~\ref{fig:Training} and then the parameters of each ring are applied to this dataset as validation.
The thickness and P2P distance after calibration are 4.4 cm and
0.47 cm.}%
\label{fig:SystematicError}%
\squeezeup
\end{figure}

\comment{
All together, 39 targets were collected as potential training data. On a separate
day, an additional 12 targets were collected for validation. Table~\ref{tab:RealExp}
reports validation data for the three calibration methods on various subsets of the
targets based on number of targets: four, seven, 10 of them.

It is seen that the trends observed in the simulation data for a spinning \lidar
continue to hold in experiments. We conclude that the proposed idea of reformulating
the calibration problem in terms of shape-preserving transformations on the observed
point cloud has been supported (i.e., not invalidated). With the open-source release
of our code and data, it will be straightforward for other groups to confirm or deny
this finding.

Informed by
the simulation study, a set of \lidar data was collected and processed by the \lidart
package~\cite{githubFileLiDARTag,huang2019lidartag}.
we placed
four targets around our \velodyne.
}

\begin{table*}[]
\caption{Consistency of the proposed methods. The numbers are the average P2P distance in meters and improvement of average P2P distance in percentage. The columns show the numbers of accumulated scans. It is seen that with higher levels of induced deterministic noise, the performance of the proposed method is invariant to the number of accumulated scans, while with lower levels of deterministic noise, the more scans included, the better is the performance. Both of these properties are desirable.}
\label{tab:ScansVsPnP}
\scalebox{0.96}{
\begin{tabular}{|c|c|c|c|c|c|c|c|c|}
\hline
                                                                  & Noise Level        & 1 scan    & 2 scans   & 4 scans   & 6 scans   & 8 scans   & 10 scans  & 12 scans  \\ \hline
Original P2P                                                      & \multirow{3}{*}{0} & 0.013998  & 0.013998  & 0.014616  & 0.014593  & 0.0144    & 0.01439   & 0.01439   \\ \cline{1-1} \cline{3-9} 
Calibrated P2P                                                    &                    & 0.0047046 & 0.0047046 & 0.0040949 & 0.0040496 & 0.0040546 & 0.0040675 & 0.0040675 \\ \cline{1-1} \cline{3-9} 
\begin{tabular}[c]{@{}c@{}}Improvement [\%]\end{tabular} &                    & 0.6639    & 0.6639    & 0.71983   & 0.7225    & 0.71844   & 0.71735   & 0.71735   \\ \hline
Original P2P                                                      & \multirow{3}{*}{2} & 0.030888  & 0.030861  & 0.029893  & 0.02983   & 0.029812  & 0.031767  & 0.031783  \\ \cline{1-1} \cline{3-9} 
Calibrated P2P                                                    &                    & 0.0040332 & 0.0040739 & 0.0040189 & 0.0039899 & 0.0039737 & 0.0040665 & 0.0040417 \\ \cline{1-1} \cline{3-9} 
\begin{tabular}[c]{@{}c@{}}Improvement  [\%]\end{tabular} &  & 0.86942 & 0.86799 & 0.86556 & 0.86625 & 0.86671 & 0.87199 & 0.87284 \\ \hline
Original P2P                                                      & \multirow{3}{*}{4} & 0.048608  & 0.047194  & 0.048691  & 0.050308  & 0.050269  & 0.050251  & 0.050254  \\ \cline{1-1} \cline{3-9} 
Calibrated P2P                                                    &                    & 0.004301  & 0.0044295 & 0.004129  & 0.0040814 & 0.0040991 & 0.0040807 & 0.004067  \\ \cline{1-1} \cline{3-9} 
\begin{tabular}[c]{@{}c@{}}Improvement  [\%]\end{tabular} &                    & 0.91152   & 0.90614   & 0.9152    & 0.91887   & 0.91846   & 0.91879   & 0.91907   \\ \hline
Original P2P                                                      & \multirow{3}{*}{6} & 0.066834  & 0.066892  & 0.066869  & 0.06685   & 0.066823  & 0.066818  & 0.066838  \\ \cline{1-1} \cline{3-9} 
Calibrated P2P                                                    &                    & 0.0043101 & 0.0044868 & 0.0044289 & 0.0044279 & 0.0044585 & 0.0044699 & 0.0044777 \\ \cline{1-1} \cline{3-9} 
\begin{tabular}[c]{@{}c@{}}Improvement  [\%]\end{tabular} &                    & 0.93551   & 0.93293   & 0.93377   & 0.93376   & 0.93328   & 0.9331    & 0.93301   \\ \hline
\end{tabular}
}
\end{table*}

\comment{
We generate deterministic measurements by applying the following
perturbations to the data on a per-ring basis (i.e, values vary by ring number): (i)
simple model $(\drho, \dtheta, \dphi)$, (ii) complex model $(\drho, \dtheta, \dphi,
s, h, v)$, and (iii) $\Sim(3)$ model $(\R, \t, s)$, denoted as $\N[1], \N[2], \N[3]$,
respectively. The exact values used are in a \texttt{.mat} file uploaded to GitHub
\cite{githubFileIntrinsic}. For each target, its true normal vector $\n$ and point
$\p[0]$ on the target are assumed known so that the P2P cost functions can be
evaluated. For each of the three cost functions \eqref{eq:OptSolution},
\eqref{eq:SimpleCalibration}, \eqref{eq:ComplexCalibration}, the \texttt{fmincon}
function in MATLAB is used to determine calibration parameters achieving a local
minimum for each Scene. The initializations for \texttt{fmincon} are selected to
correspond to zero perturbation error, that is the parameters are zero, one or the
identify matrix as the case requires. The optimized calibration parameters are then
applied to the data of Scene 5 and the P2P cost of \eqref{eq:OptSolution} is
evaluated and reported as validation data in Table~\ref{tab:N3}.

The simulated results confirm that a single target results in degeneracy when
calibrating with the $\BL_2$ and $\Sim(3)$ models, and while $\BL_1$ works fine with
a single target, it drops in performance as the complexity of the calibration errors
is increased. The well-posedness of the $\BL_2$ and $\Sim(3)$ calibration models was
confirmed when the conditions of Theorem.~\ref{thm:uniqueness} are respected.

\addd{To further analyze the proposed method and verify Theorem~\ref{thm:uniqueness},
    we consider the following two scenarios: \textbf{1)} We generate \bh{10 thousand}
    training scenes in which we randomly place 4 targets at different
    distances with various angles. \textbf{2)} We generate 10 thousand training
    scenes in which we randomly rotate the tetrahedron, as shown in
    Fig.~\ref{fig:oriented-tetrahedron}. We compute the condition number of normal
    vectors (Assumption~\ref{asm:assumtionN}), and the condition number of bases
    (Assumption~\ref{asm:assumtionB}) for both scenarios. The optimized calibration
    parameters are then applied to the data of Scene 5 and the P2P cost of
\eqref{eq:OptSolution} is evaluated. The results are reported in
Fig.~\ref{fig:P2PResults}. Again, we can
see, as expected, when the condition number of normal vectors
worsens, the validation P2P cost increases. However, when the condition number of bases increases, the
P2P cost does not increase as we expected. Therefore, in practice, 
a set of well-conditioned target's normal vectors is sufficient to calibrate a
\slidarN.}
}

\begin{figure}[t]%
\centering
\subfloat{%
\label{fig:SSLiDARInitial}%
\includegraphics[height=0.23\textheight, trim={0cm 0cm 0cm 0cm},clip]{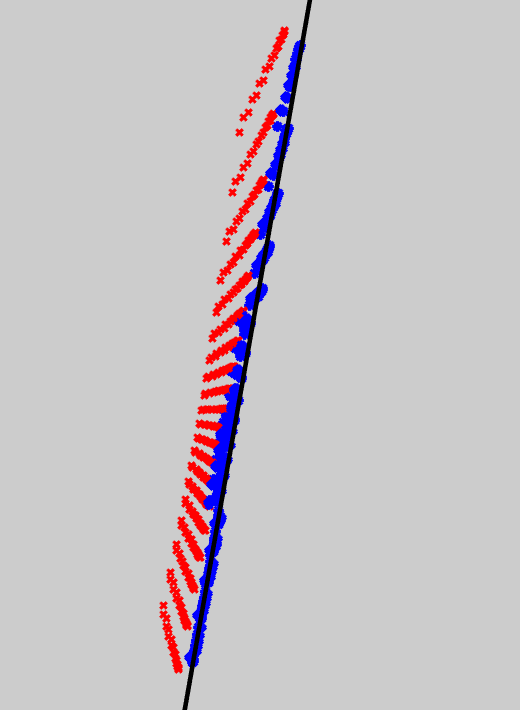}}~
\subfloat{%
    \label{fig:SSLiDARCalibrated}%
\includegraphics[height=0.23\textheight, trim={0cm 0cm 0cm 0cm},clip]{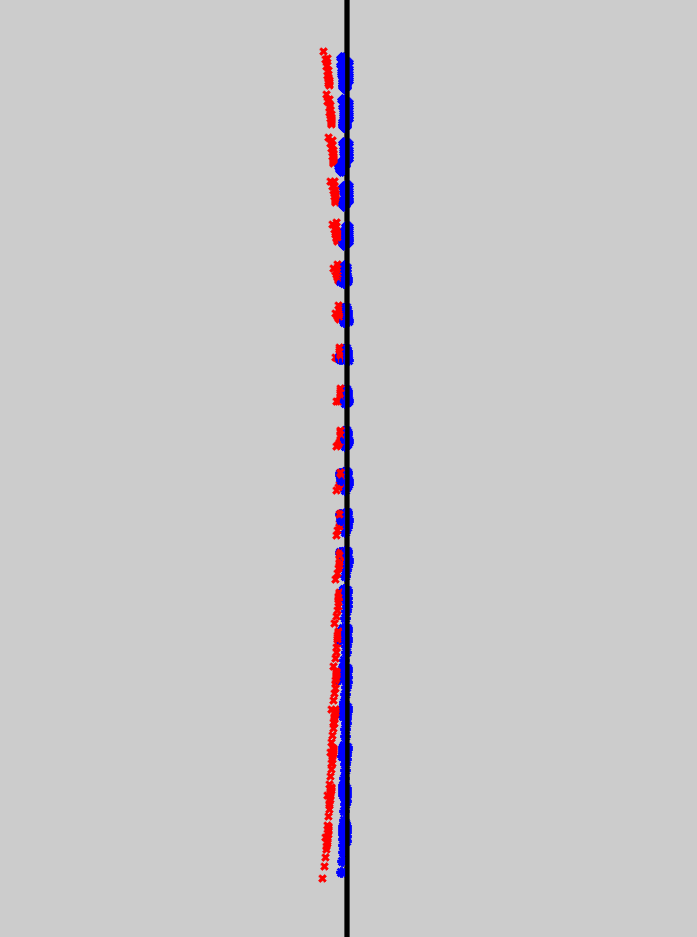}}
\caption[]{A \sslidar is (hypothetically) built on a warped wafer and measures a planar target (black). The un-calibrated measurements
and calibrated measurements are marked in red and blue, respectively. The left image shows a top-down view and the right image shows a side view of a target. The proposed method reduces the P2P distance by
48.7\%.}
\label{fig:SSLiDAR}
\squeezeup
\end{figure}

\squeezeup
\subsection{\lidar Simulated Intrinsic Calibration of Solid-State \lidars}
\label{sec:ExpSSLiDAR}
At the time of writing, we do not have access to a \sslidarN; therefore only simulation results are provided. However, we are open to collaborating if the reader can share data. OPA \sslidars are fabricated on a planar wafer with a large number of emitters, see
Sec.~\ref{sec:intro}. For such a \lidarN, we induce geometric uncertainty into a simulation of the system by assuming the
plane of the wafer is slightly warped. In the simulator, the OPA \sslidar has
$20\times20$ emitters with $160^\circ\times40^\circ$ (horizontal, vertical) field of view. We place four targets of the same size in the simulator.

Theorem~\ref{thm:uniqueness} requires a ring plane, as mentioned in Sec.~\ref{sec:TargetPlacementGuideline}. We therefore parse the \lidar returns into $20\times4$ uniform grids over the planar targets, with each grid containing five points. Each grid plays the role of a ring plane in Theorem~\ref{thm:uniqueness}, resulting in a total of 80 calibration models. We use Weighted Least Squares to estimate the plane of the target. The calibrated sensor is then validated on four other scenes that have different angles and distances. The mean P2P cost of \eqref{eq:OptSolution} improves by 48.7\%; the result is visualized in Fig.~\ref{fig:SSLiDAR}.

\begin{remark}
The \sslidar is used to measure the same target four times with different orientations. The four target planes form a tetrahedron. Therefore, the same five points can scan the four planes and result in a ring plane. Instead of five points, we can group a row/column of points, as mentioned in Sec.~\ref{sec:CollectionOfPoints}.
\end{remark}

\section{Conclusions}
\label{sec:Conclusions}
We proposed a universal method for LiDAR intrinsic calibration that abstracts away
the physics of a \lidar type (spinning head vs. solid-state, for example) and focuses
instead on the spatial geometry of the point cloud generated by the sensor. The
calibration parameter becomes an element of $\Sim(3)$, a matrix Lie group. We
mathematically prove that given four targets with appropriate orientations, the
proposed model is well-constrained (i.e., a unique answer exists). Because $\Sim(3)$
is closely related to $\SE(3)$, we showed how to profitably apply efficient, globally
convergent algorithms for $\SE(3)$ to determine a solution to our problem in
$\Sim(3)$. The resulting algorithm was evaluated in simulation for a solid-state
LiDAR and simulation and experiment for a spinning LiDAR. The P2P distance of the validation
scene for the \slidar was reduced 44.7\% and 68.6\% in simulation and
experiment, respectively. The P2P distance of the validation scene for the \sslidar
was reduced by 48.7\%. Both simulation and experiments showed that the proposed method
can serve as a generic model for intrinsic calibration of both spinning and
solid-state \lidarsN.

\comment{
and nicely generalized previous parametrization
models (from 3-parameter to 10-parameter models) into a unifying problem that can
solve both \slidars and \sslidarsN. We also mathematically proved that when solving
a transformation in $\Sim(3)$, a minimum of four appropriately orientated targets are
required to ensure a unique answer exists. The results of the proof provide a
guideline for target positioning where the four target should form a tetrahedron.
We introduced a global algorithm to find the unique answer of the intrinsic
calibration problem. The algorithm utilizes the bisection algorithm to solve the
scaling parameter and the rest of the problem is formulated as QCQP where Lagrangian
duality is used. The relaxed problem becomes an SDP and convex. Therefore, the
problem can be solved efficiently and globally.

\add{
We developed a novel LiDAR intrinsic calibration method that uses the spatial
geometry of the point cloud generated by the sensor. The calibration parameter was
modeled using $\Sim(3)$ Lie group and nicely generalized previous parametrization
models (from 3-parameter to 10-parameter models) into a unifying problem that can
solve both \slidars and \sslidarsN. We also mathematically proved that when solving
a transformation in $\Sim(3)$, a minimum of four appropriately orientated targets are
required to ensure a unique answer exists. The results of the proof provide a
guideline for target positioning where the four target should form a tetrahedron.
We introduced a global algorithm to find the unique answer of the intrinsic
calibration problem. The algorithm utilizes the bisection algorithm to solve the
scaling parameter and the rest of the problem is formulated as QCQP where Lagrangian
duality is used. The relaxed problem becomes an SDP and convex. Therefore, the
problem can be solved efficiently and globally.
}

\add{
We extensively validated this finding both in simulation and experiments. The P2P
distance of the validation scene for the \slidar was reduced \bh{44.7\%} and
\bh{68.6\%} in simulation and experiments, respectively. The P2P distance of the
validation scene for the \sslidar was reduced 48.7\%.
Both simulation and experiments showed that the proposed method can serve as a
generic model for intrinsic calibration of both spinning and solid-state \lidarsN. 
}
}

\section*{Acknowledgment}
Toyota Research Institute (TRI) provided funds to support this work. Funding for J. Grizzle was in part provided by TRI and in part by NSF
Award No.~1808051. The first author thanks Wonhui Kim for useful conversations. The
authors thank Carl Olsson for kindly providing the experimental datasets
in~\cite{olsson2008solving}.

\balance

\bibliographystyle{bib_/IEEEtran}
\bibliography{newbib/strings-abrv,newbib/ieee-abrv,newbib/BipedLab.bib, newbib/Books.bib, newbib/Bruce.bib, newbib/ComputerVision.bib, newbib/ComputerVisionNN.bib, newbib/IntrinsicCal.bib, newbib/L2C.bib, newbib/LibsNSoftwares.bib, newbib/ML.bib, newbib/OptimizationNMath.bib, newbib/Other.bib, newbib/StateEstimationSLAM.bib, newbib/ComputerGraphics.bib}

\begin{appendices}
\newpage
\setcounter{theorem}{0}
\setcounter{assumption}{0}
\setcounter{assumptionB}{1}
\setcounter{assumptionN}{13}
\section{Proof on Uniqueness of Similarity Transformation}
\label{sec:Uniqueness}
The proof provides a guideline to place targets such that an answer to the optimization
problem is unique. In particular, under Assumption~\ref{asm:assumtionN} and
Assumption~\ref{asm:assumtionB}, the optimization problem has
only one unique answer. As stated in Sec.~\ref{sec:IntrinsicCalibration}, the
intrinsic calibration parameters are modeled as an element of the similarity Lie
group ($\Sim(3)$). An element of this group in matrix from is
\begin{equation}
    \H = 
    \begin{bmatrix}
        s\R & \t \\
        \zeros & 1
    \end{bmatrix} \in \Sim(3),
\end{equation}
where $\R\in\SO(3), \t \in \reals^3$ and $s \in \reals^+$.

\subsection{Mathematical Definitions and Preliminaries}
Let $[\Scal]$ and $[\Scal]^\perp$ be the span and its orthogonal complement of
$\Scal\subset \realnumbers^3$, respectively\footnote{$\because \Scal \subset
\real^3$, $[\Scal]=\real^3 \iff [\Scal]^\perp =0$}. We denote the union of
intersection of $K$ spans as $\left( [\Scal_1] + [\Scal_2] + \cdots +
[\Scal_K]\right) = \cap_{i=1}^K[\Scal_i]$. Let $\P \subset \real^3$ be a plane.
$\P$ modeled the set traced out by a single ring of a perfectly calibrated
spinning \lidarN. Let $\{\e[1], \e[2], \e[3] \}$ denote the canonical basis for
$\real^3$. Without loss of generality, we assume $\P=[\e[3]]^\perp = \{\e[1],
\e[2]\}$. Therefore, $\forall s>0, \R\in\SO(3)$, 
$$ (s\R -\I)(\P)=0 \iff s=1, \R=\I.
\footnote{$(s\R-\I)(\P) = [(s\R-\I)\e[1]] + [(s\R-\I)\e[2])]$}$$

\subsection{Assumptions}
Consider four targets with unit normal vectors $\n[i]$ and let $\p[0, i]$ be a points
on the $i$-th target. 

\begin{assumptionN}[Normal Vectors]
    \label{eq:assumtionN}
    All sets of three distinct vectors from $\{ \n[1], \n[2], \n[3], \n[4], \e[3]\}$
    are \li.
\end{assumptionN}

For $1\le i \le 4$, the plane defined by the $i$-th target is $\V[i]:= \p[0,i] +
[\n[i]]^\perp.$ Under Assumption~\ref{asm:assumtionN}, some facts are listed below
without proofs:

\begin{enumerate}[(a)]
    \item For each $i \neq j \in \{1, 2, 3, 4 \}$, $\p[ij]:=\P \cap \V[i] \cap \V[j]$
        exists and is unique. There are $\binom{4}{2}$ intersection points.

    \item $\V[i] \cap \V[j]=\p[ij] + [\n[i], \n[j]]^\perp$.

    \item  $\H\cdot \p[ij] \in \V[i] \cap \V[j]$ if, and only if, $$(s \R-\I)\p[ij]
        +\t \in [\n[i], \n[j]]^\perp. $$

    \item Because $\dim [\n[i], \n[j]]^\perp=1$,  $(s \R-\I)\p[ij] +\t  \neq 0$ if,
        and only if, $ [(s \R-\I)\p[ij] +\t] = [\n[i], \n[j]]^\perp. $

    \item Let $\{\p[a], \p[b] \}$ be a basis for $\P$. Then $(s \R-\I)(\P)=0$ if,
        and only if, $(s \R-\I)(\p[a])=0$ and $(s \R-\I)(\p[b])=0$.
\end{enumerate}
\vspace{3cm}

\begin{assumptionB}[Basis Vectors]
    \label{eq:assumtionB}
    Given the six intersection points, any two form a set of basis for the ring
    plane. There are $\binom{6}{2}$ sets of basis. Subsets of the basis' are \li.
    \begin{enumerate}[(a)]
        \item $\{ \p[12], \p[13] \}$, $\{ \p[13], \p[14] \}$, $\{ \p[14], \p[12] \}$,
        \item $\{ \p[12], \p[23] \}$, $\{ \p[23], \p[24] \}$, $\{ \p[24], \p[12] \}$,
        \item $\{ \p[13], \p[23] \}$, $\{ \p[23], \p[34] \}$, $\{ \p[34], \p[13] \}$,
        \item $\{ \p[14], \p[24] \}$, $\{ \p[24], \p[34] \}$, $\{ \p[34], \p[14] \}$,
        \item $\{ \p[14], \p[23] \}$
    \end{enumerate}
\end{assumptionB}

\subsection{A Complete Proof}
\begin{theorem} 
    Assume that Assumptions~\ref{asm:assumtionN} and \ref{asm:assumtionB} hold and let $\H \in \Sim(3)$. If for each $i
    \neq j \in \{1, 2, 3, 4 \}$, $\H\cdot\p[ij] \in  \V[i] \cap
     \V[j]$, then 
     \begin{equation} 
         \H = 
         \begin{bmatrix}
            \I & \zeros \\
            \zeros & 1
         \end{bmatrix}.
     \end{equation}
 \end{theorem}

\begin{proof} 
The proof is by exhaustion on the dimension of $(s\R-\I)(\P)$. We have that
$\H\cdot\p[ij] = s\R \p[ij] + \t$, and therefore 
\begin{equation}
\label{eq:NecssSuff}
   \H\cdot\p[ij] \in  \V[i] \cap \V[j] \iff (s\R-\I) \p[ij] + \t \in [\n[i], \n[j]]^\perp. 
\end{equation}
 
\begin{itemize}

\item \textbf{Case 1:} $\dim (s\R-\I)(\P) =0$. \par
    Then $\t \in [\n[i],
\n[j]]^\perp$ for all $i\neq j$. Hence, $\t \in [\n[1], \n[2]]^\perp \cap
[\n[2], \n[3]]^\perp =0$, which implies that $\t=0$. Therefore, by
\textbf{Simple Fact 1}, we are done. We next show that $\dim (s\R-\I)(\P) > 0$
and $\H\cdot\p[ij] \in  \V[i] \cap \V[j]$ for all $i \ne j$ lead to
contradictions.

\item  \textbf{Case 2:} $\dim (s\R-\I)(\P) =1$ \par
    (a) Suppose $\t \not \in (s\R-\I)(\P)$. Then, for any $\p[ij]$, $(s\R-\I)\p[ij] +
    \t \neq 0$. Hence, 
    \begin{align*}
     [\n[1], \n[2]]^\perp &= [(s\R-\I)(\p[12]) + \t] \\
     [\n[2], \n[3]]^\perp &= [(s\R-\I)(\p[23]) +\t] \\
      [\n[3], \n[4]]^\perp &= [(s\R-\I)(\p[34]) +\t].
    \end{align*}
Because $[\n[1], \n[2]] \cap [\n[2], \n[3]]\cap [\n[3], \n[4]] =0$, we have
that $[\n[1], \n[2]]^\perp + [\n[2], \n[3]]^\perp +  [\n[3], \n[4]]^\perp =
\real^3$. We deduce that $\dim (s\R-\I)(\P) =2$, which is a contradiction. \\

(b) Suppose $\t \in (s\R-\I)(\P)$ and thus there exists $\p[t]\in \P$ such that
$\t=(s\R-\I)(\p[t])$. The condition \eqref{eq:NecssSuff} can therefore be written as
\begin{equation}
\label{eq:NecssSuffWithT}
   \H\cdot\p[ij] \in  \V[i] \cap \V[j] \iff (s\R-\I)( \p[ij] + \p[t]) \perp \{\n[i],\n[j] \}. 
\end{equation}

  Because  $\{ \p[12], \p[13] \}$,  $\{ \p[13], \p[14] \}$,  $\{ \p[14], \p[12] \}$
  are bases for $\P$, we conclude that 
  \begin{equation}
      [\p[12]+\p[t], \p[13]+\p[t], \p[14]+\p[t] ]=\P.
  \end{equation}
  Applying \eqref{eq:NecssSuffWithT} to this set of vectors, we deduce that
  $$(s\R-\I)(\P) \perp \n[1].$$ 
  Applying the same reasoning to $\{ \p[12], \p[23] \}$,
  $\{ \p[23], \p[24] \}$, $\{ \p[24], \p[12] \}$ , we deduce that $(s\R-\I)(\P) \perp
  \n[2]$. Repeating again, we have $(s\R-\I)(\P) \perp \n[3]$ and thus $\dim
  (s\R-\I)(\P)=0$.

\item  \textbf{Case 3:} $\dim (s\R-\I)(\P) =2$ \par
    We rewrite \eqref{eq:NecssSuff} as
    \begin{equation}
    \label{eq:NecssSuffv02}
       \H\cdot\p[ij] \in  \V[i] \cap \V[j] \iff -\t \in  (s\R-\I) \p[ij] +[\n[i], \n[j]]^\perp. 
    \end{equation}
Hence, if for all  $i \neq j \in \{1, 2, 3, 4 \}$, $\H\cdot\p[ij] \in  \V[i] \cap
\V[j]$, then the lines $(s\R-\I) \p[ij] +[\n[i], \n[j]]^\perp$ in $\real^3$ must have
a common point of intersection, namely $-\t$ and hence
\begin{equation} 
\label{eq:MasternIntersection}
\bigcap \limits_{i \neq j \in \{1, 2, 3, 4 \}} \{ (s\R-\I) \p[ij] + [\n[i], \n[j]]^\perp \} \neq \emptyset.
\end{equation}

\begin{remark}  
    When $(s\R-\I)(\P)=0$, the intersections in \eqref{eq:MasternIntersection} are
    non-empty; indeed, the equation reduces to 
    $$\bigcap \limits_{i \neq j \in \{1,
    2, 3, 4 \}} [\n[i], \n[j]]^\perp = \left(~ \sum \limits_{i \neq j} [\n[i], \n[j]]
    ~\right)^\perp = (\real^3)^\perp =0,$$
and hence $\t=0$.
\end{remark}

In the remainder of the proof, we show the intersection being non-empty
contradicts ${\rm dim}~(s\R-\I)(\P) = 2$. We do this by examining the
intersections in \eqref{eq:MasternIntersection} pairwise to arrive at a set of
necessary conditions for $\H \cdot \p[ij] \in  \V[i] \cap \V[j]$ for $i \neq
j$, and then use the necessary conditions to complete the proof. We note that for $ij
\neq kl$, 
  
\begin{align*}
 &\{ (s\R-\I) \p[ij] + [\n[i], \n[j]]^\perp \} \cap \{ (s\R-\I) \p[kl] +
[\n[k], \n[l]]^\perp \} \neq \emptyset \\
 &\iff (s\R-\I)(\p[ij] -  \p[kl]) \in
[\n[i], \n[j]]^\perp  + [\n[k], \n[l]]^\perp.
\end{align*}
 
 The indices are a bit easier to keep track of if we set 
 \begin{align}
 &\q[1]:=(s\R-\I)(\p[12]), ~\U[1]:=[\n[1], \n[2]]^\perp \nonumber\\
 &\q[2]:=(s\R-\I)(\p[13]), ~\U[2]:=[\n[1], \n[3]]^\perp \nonumber\\
 &\q[3]:=(s\R-\I)(\p[23]), ~\U[3]:=[\n[1], \n[4]]^\perp \nonumber\\
 &\q[4]:=(s\R-\I)(\p[14]), ~\U[4]:=[\n[2], \n[3]]^\perp \nonumber,
 \end{align}
 where $\{\U[k]|k=1,\cdots,4\}$ denote the indicated one-dimensional subspaces. Then,
 for each $i \neq j \in \{1, 2, 3, 4 \}$, $\U[i] \cap \U[j] =0$, and we have $\U[1]
 \oplus \U[2] \oplus \U[3]=\real^3$. Let $\u[i]$ be a basis for $\U[i]$, so that
 $\U[i] = [\u[i]]$, and write $$\u[4] =\alpha \u[1] + \beta \u[2] + \gamma \u[3].$$

\begin{claim} 
\label{claim:abg}
Each of the coefficients $\alpha, \beta, \gamma$ is non-zero.
\end{claim}
\begin{proof}  Suppose $\alpha=0$. Then $\U[4] \subset \U[2] + \U[3]$, that is,
$$[\n[2], \n[3]]^\perp  \subset [\n[1], \n[3]]^\perp +[\n[1], \n[4]]^\perp $$ But this is equivalent to
$$ [\n[1]] = [\n[1], \n[3]] \cap [\n[1], \n[4]]  \subset [\n[2], \n[3]], $$ and hence $\{
\n[1], \n[2], \n[3]\}$ is not linearly independent, contradicting Assumption~\ref{asm:assumtionN}. The same argument holds for the other coefficients. 
 \end{proof}
 
 \begin{claim}  A necessary condition for 
\begin{equation} 
\label{eq:MasternIntersection2}
\bigcap \limits_{i=1}^{4}  \{\q[i] + \U[i] \} \neq \emptyset
\end{equation}
is that there exist real numbers $c_1, c_2, c_3, c_4$ such that
\begin{equation}
\label{eq:SimulEqns}
\begin{aligned}
    \Delta \q[12]:=\q[1]-\q[2]&=c_1 \u[1] + c_2\u[2] \\
    \Delta \q[13]:=\q[1]-\q[3]&=c_1 \u[1] + c_3\u[3] \\
    \Delta \q[14]:=\q[1]-\q[4]&=c_1 \u[1] + c_4 \u[4] \\
    \Delta \q[23]:=\q[2]-\q[3]&=c_3 \u[3] - c_2\u[2] \\
    \Delta \q[24]:=\q[2]-\q[4]&=c_4 \u[4] - c_2\u[2] \\
    \Delta \q[34]:=\q[3]-\q[4]&=c_3 \u[3] - c_4\u[4] \\
\end{aligned}
\end{equation}
\end{claim}

\begin{proof}
Each row of \eqref{eq:SimulEqns} corresponds to a condition of the form $\q[i]-\q[j] \in \U[i] \oplus \U[j]$. 
The proof proceeds by expressing each of the six rows in \eqref{eq:SimulEqns} with distinct coefficients (12 in total), and then writing down three necessary compatibility conditions, 
\begin{equation}
\label{eq:CompatibilityEqns}
\begin{aligned}
\Delta \q[12]-\Delta \q[13]+\Delta \q[23]&=0 \\
\Delta \q[12]-\Delta \q[14]+\Delta \q[24]&=0\\
\Delta \q[14]-\Delta \q[13]+\Delta \q[34]&=0.
\end{aligned}
\end{equation}
Because $-\t$ is in the intersection of the lines in \eqref{eq:MasternIntersection2},
the resulting linear equations must have a solution, and indeed direct computation
shows that the set of solutions can be parameterized as given in the claim.
\end{proof}

The next step is to note that $[\Delta \q[12], \cdots, \Delta \q[34] ] \subset
(s\R-\I)(\P)$, and hence its dimension must be less than three. Additional
straightforward calculations show that
$$\dim ~[\Delta \q[12], \Delta \q[13], \Delta \q[34] ] =  {\rm rank~} 
\begin{bmatrix}
    c_1 & c_1 & -\alpha c_4 \\ c_2 & 0 & c_2 - \beta c_4 \\ 0 & c_3 &-\gamma c_4
\end{bmatrix}.
$$

In light of Claim~\ref{claim:abg}, the rank is less than three if, and only if, any
two coefficients of $\{ c_1, c_2, c_3, c_4\}$ are zero. But if this is the case, then
at least one row of \eqref{eq:SimulEqns} must be zero. Each row of
\eqref{eq:SimulEqns}, however, has the form $(s\R-\I)(\p[a] - \p[b])$, where
$\{\p[a], \p[b] \}$ is a basis for $\P$, and thus it cannot be the case that $\dim
(s\R-\I)(\P) = 2.$ This completes the proof. 
\end{itemize}
\end{proof}

\newpage

\section{Lagrangian Duality Relaxation \\for P2P Distance on $\SE(3)$}
\label{sec:DualFormulation}
At the $k$-th iteration, the scaling parameter is determined (see
Sec.~\ref{sec:ConvexRelaxation}) and the rest of the parameters are $\SE(3)$. To
solve the remaining parameters, we adopt techniques that were used to solve 3D
registration or 3D SLAM~\cite{tron2015inclusion, carlone2015lagrangian,
olsson2008solving, briales2017convex}. We summarize below for completeness.
The action of $\SE(3)$ on $\reals^3$ can be rewritten as:
\begin{equation}
    \H\cdot\x = \R\x+\v =
    \begin{bmatrix}
        \x[][\top] \otimes I_3 & I_3
    \end{bmatrix}
    \underbrace{\begin{bmatrix}
    \vect{\R} \\\v
    \end{bmatrix}}_{\tau},
\end{equation}
where $\otimes$ and $\vect{\cdot}$ are the Kronecker
product~\cite{brewer1978kronecker} and the vectorization
operation~\cite{NotesOnMatrix}, respectively. 

\subsection{P2P Distance Reformulation and Quadratic Formulation}
The P2P distance~\eqref{eq:P2PDistance} can be equivalently reformulated into a
quadratic form:
\begin{align}
    \label{eq:QuadraticForm1}
    \sum_{i=1}^M J_i &:= \sum_{i=1}^M|(\H\cdot \x[i]-\p[0,t])^\top (\n[t]\n[t][\top]) (\H\cdot\x[i]\p[0,t])|\nonumber\\
                     &= \sum_{i=1}^M\begin{bmatrix}\tau\\1\end{bmatrix}^\top N_i^\top
                     (\n[t]\n[t][\top]) N_i \begin{bmatrix}\tau\\1\end{bmatrix} =
                 \tilde{\tau}^\top W_t \tilde{\tau},
\end{align}
where $N_i = \begin{bmatrix} \x[i]\otimes I_3|-\p[0,t]\end{bmatrix}$ and
$W_t=\sum_{i}^{M}N_i^\top(\n[t]\n[t][\top])N_i$.  After rearranging
\eqref{eq:QuadraticForm1}, the resulting problem in \eqref{eq:OptSolution} becomes
\begin{align}
    \label{eq:QuadraticForm2}
    &\min_{H\in\SE(3)} \sum_{t=1}^T
    \begin{bmatrix}
        \vect{\R}\\1\\\v
\end{bmatrix}^\top
\tilde{W}_t^\prime
\begin{bmatrix}
    \vect{\R}\\1\\\v
\end{bmatrix}\nonumber\\
&= \min_{H\in\SE(3)} \sum_{t=1}^T
    \begin{bmatrix}
        \tilde{\r}\\\v
\end{bmatrix}^\top
\begin{bmatrix}
    \tilde{W}_{\tilde{\r},\tilde{\r}} & \tilde{W}_{\tilde{\r},\v}\\
    \tilde{W}_{\v,\tilde{\r}} & W_{\v,\v}
\end{bmatrix}_t
    \begin{bmatrix}
        \tilde{\r}\\\v
\end{bmatrix}\nonumber\\
&=\min_{H\in\SE(3)}\tilde{\r}^\top\tilde{W}_{\tilde{\r},\tilde{\r}}\tilde{\r} + 2\v[][\top]\tilde{W}_{\v,\tilde{\r}}\tilde{\r}+\v[][\top] W_{\v,\v} \v
\end{align}
where $\tilde{\r} = \begin{bmatrix}\vect{\R}^\top&1\end{bmatrix}^\top$. We then introduce the Lagrangian
multipliers. Due to $R\in\SO(3)$ constraints, the derivative with respect to $\v$ is
zero: $\partial \mathrm{L}(H, \lambda)/\partial \v=0$, which leads to
\begin{equation}
\label{eq:optimalT}
\v[][*] = -\inv{(W_{\v,\v})}\tilde{W}_{\v,\tilde{\r}}\tilde{\r}.
\end{equation}
By substituting \eqref{eq:optimalT} into \eqref{eq:QuadraticForm2}, we have
\begin{equation}
    \label{eq:QuadraticForm3}
    \min_{\R\in\SO(3)} \tilde{\r}^\top \tilde{\Q} \tilde{\r},
\end{equation}
where $\tilde{\Q}$ is the Schur complement of $\tilde{W}_t$ and equal to
$\tilde{W}_{\tilde{\r},\tilde{\r}}-\tilde{W}_{\tilde{\r},\v}W^{-1}_{\v,\v}\tilde{W}_{\v,\tilde{\r}}$.

\subsection{Primal Problem and Its Dual}
From~\cite{tron2015inclusion,carlone2015lagrangian,briales2017convex}, we re-define
\eqref{eq:QuadraticForm3} to an \textit{equivalent, homogeneous, strengthened} primal
problem:
\begin{align}
    \label{eq:Primal}
    &\min_{R}f(\tilde{\q}) = \min_{R} \tilde{\q}^\top \tilde{\Q}\tilde{\q},~\tilde{\q}=\begin{bmatrix}\vect{\R}^\top
    & y\end{bmatrix}^\top\\
    \text{s.t.~~}&R^\top R=y^2 I_3\nonumber\\
    &RR^\top=y^2I_3\nonumber\\
    &R^{(i)}\times R^{(j)}=yR^{(k)},~i,j,k=\text{permute}\{1,2,3\}\nonumber\\
    &y^2=1\nonumber.
\end{align}
The primal problem \eqref{eq:Primal} is a QCQP and the corresponding dual
problem is defined as
\begin{equation}
    \mathrm{L}(\tilde{\q}, \boldsymbol \lambda) = \gamma + \tilde{\q}^\top (\tilde{\Q}
    + \tilde{\P}(\boldsymbol \lambda))\tilde{\q} = \gamma + \tilde{\q}^\top \Z\tilde{\q},
\end{equation}
where $\P$ is the penalization matrix~\cite{briales2017convex}. The Lagrangian
relaxation is an unconstrained problem and has a closed-form solution:
\begin{align}
    \label{eq:Dual}
    g(\boldsymbol \lambda)  &= \min_{\tilde{\q}}\mathrm{L}(\tilde{\q}, \boldsymbol
    \lambda) = \min_{\tilde{\q}} \gamma + \tilde{\q}^\top \Z\tilde{\q}\\
                            &= \begin{cases}
                                \gamma, \text{~if~} \Z \succeq 0\\
                                -\infty, \text{~otherwise}.
                            \end{cases}
\end{align}
Therefore, the maximization of the dual problem \eqref{eq:Dual} is a SDP:
\begin{equation}
\label{eq:Dual2}
    g^* = \max_{\boldsymbol \lambda} \gamma, \text{~s.t.~} \Z(\boldsymbol \lambda)\succeq 0
\end{equation}
This problem is convex and can be solved globally by off-the-shelf specialized
solvers~\cite{grant2014cvx}. It is shown in \cite{briales2017convex} that the
relaxation is empirically always tight (the duality gap is zero).

\begin{figure*}[b]%
\centering
\subfloat{%
    \label{fig:RandomBefore}%
\includegraphics[width=0.3\textwidth, trim={0cm 0cm 0cm 0.1cm},clip]{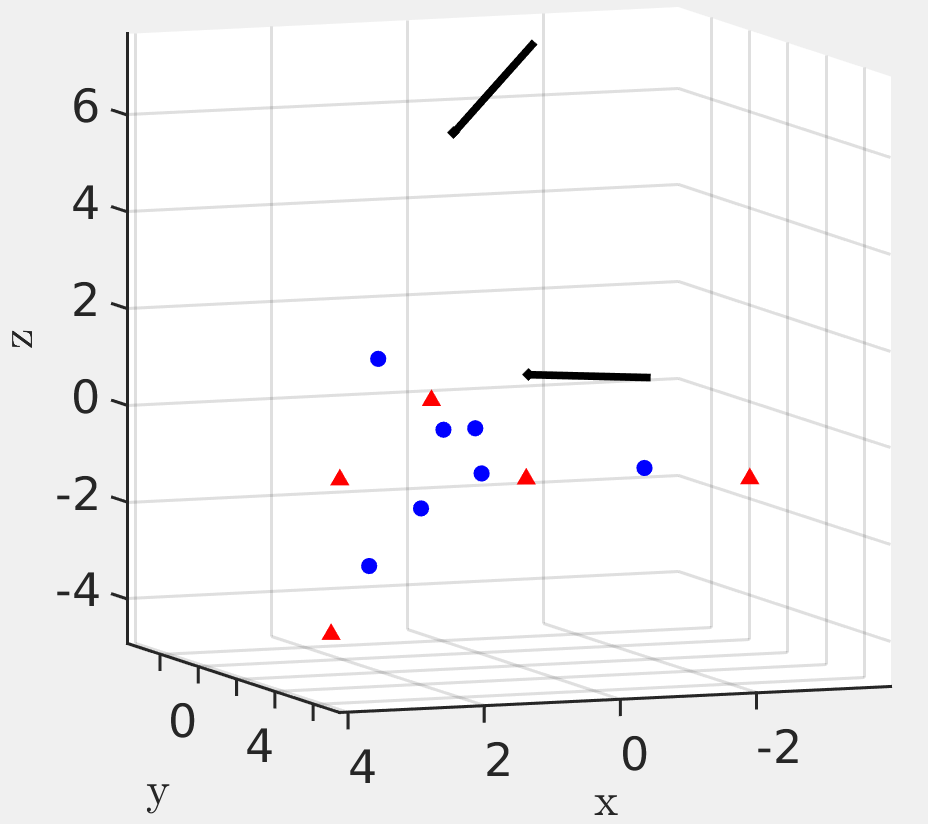}}~
\subfloat{%
\label{fig:RubikCubeBefore}%
\includegraphics[width=0.3\textwidth, trim={0cm 0cm 0cm 0cm},clip]{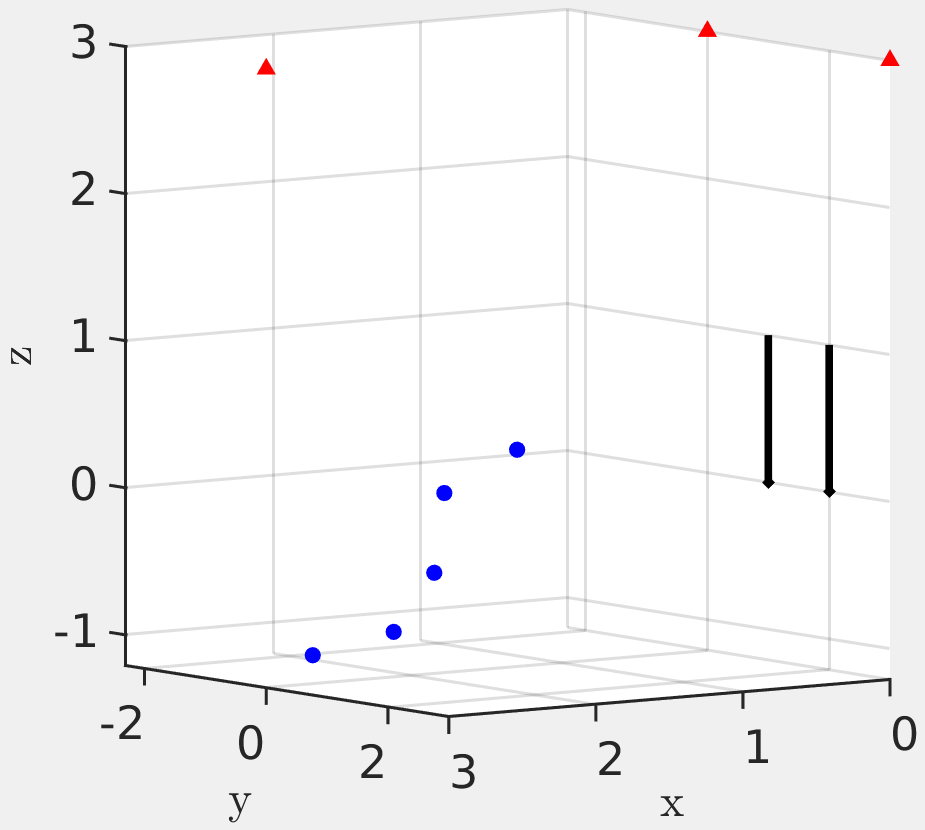}}~
\subfloat{%
    \label{fig:SpaceStationBefore}%
\includegraphics[width=0.3\textwidth, trim={0cm 0cm 0cm 0cm},clip]{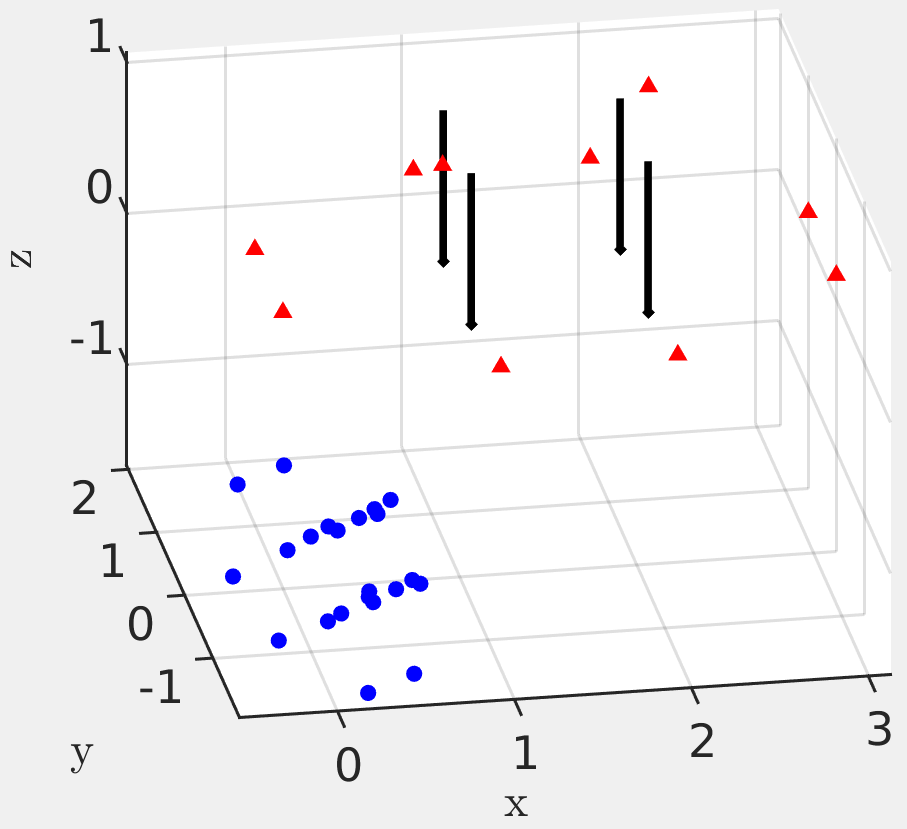}}\\
\subfloat{%
    \label{fig:RandomAfter}%
\includegraphics[width=0.3\textwidth, trim={0cm 0cm 0cm 0cm},clip]{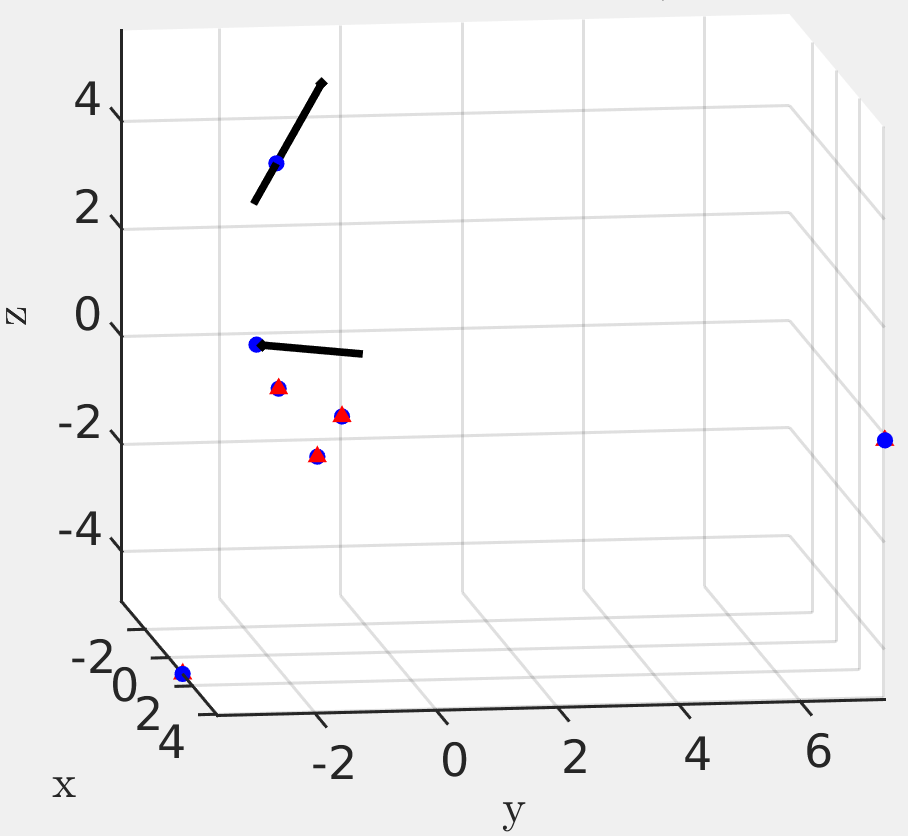}}~
\subfloat{%
    \label{fig:RubikCubeAfter}%
\includegraphics[width=0.3\textwidth, trim={0cm 0cm 0cm 0cm},clip]{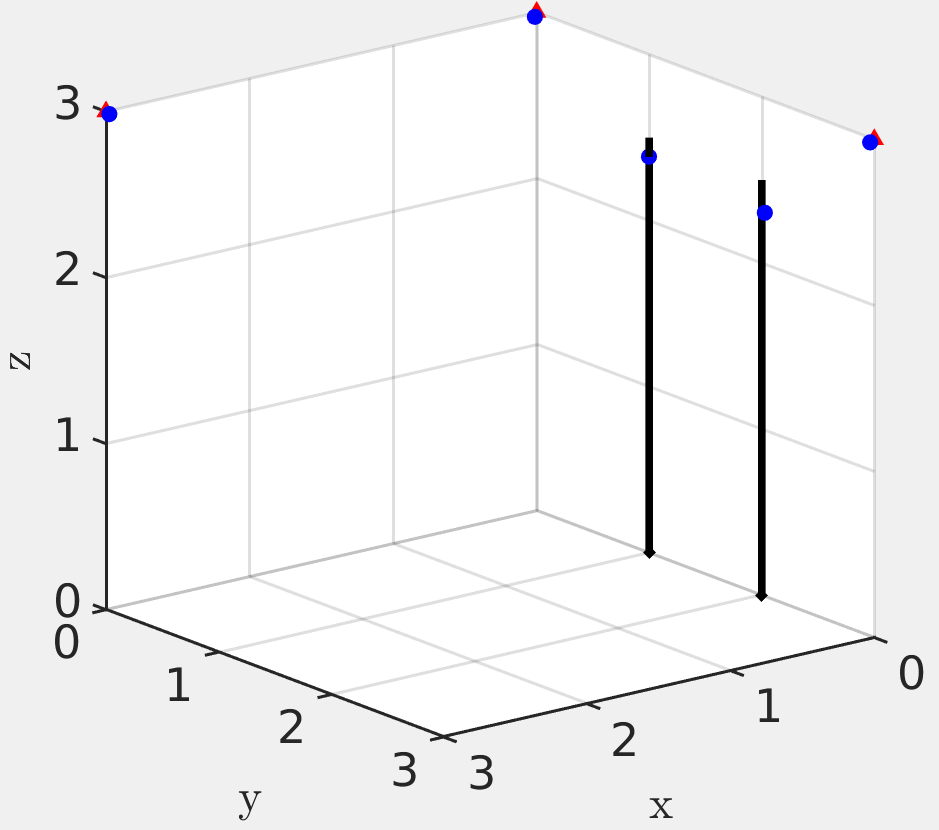}}~
\subfloat{%
    \label{fig:SpaceStationAfter}%
\includegraphics[width=0.3\textwidth, trim={0cm 0cm 0cm 0.2cm},clip]{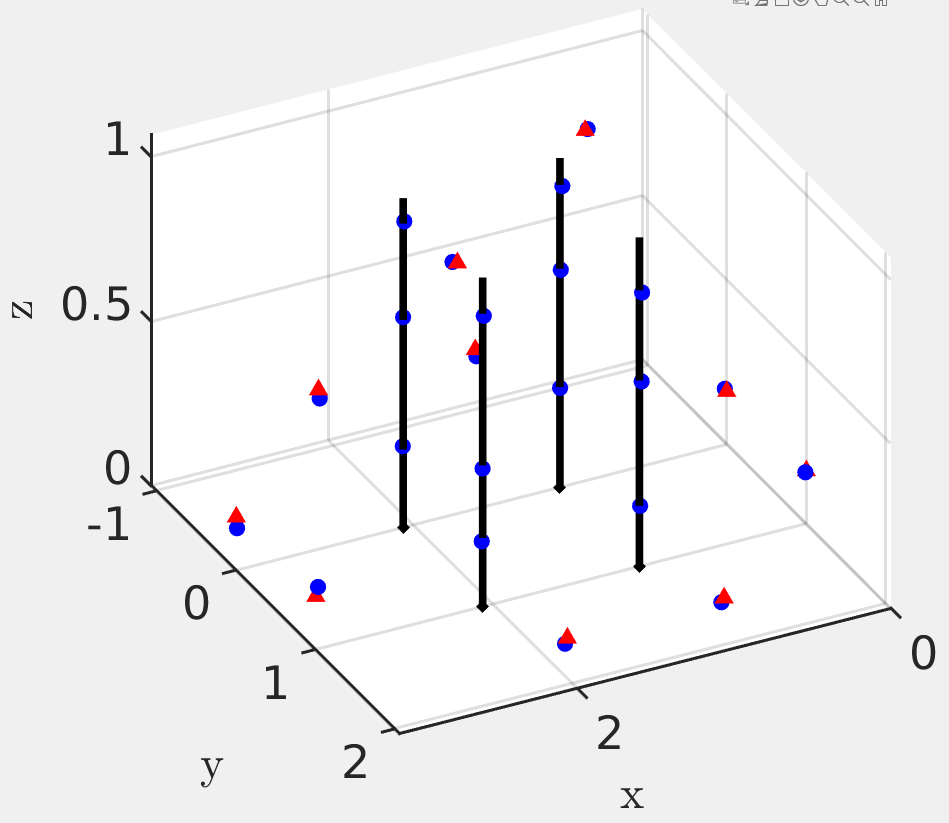}}
\caption[]{This figure shows the proposed method can be used in 3D registration
problems. The correspondences are point-to-point (marked in blue and red),
point-to-line (marked in blue and black), and point-to-plane (not shown to keep
the figure readable). The top and bottom row is the initial status
and the results of the registration problems, respectively. The simulation (Random)
with five point-to-point and two point-to-line correspondences are shown in the left
column. The result of the experimental data (RubikCube) with three point-to-point and
two point-to-plane correspondences is shown in the second column. The last column
shows the result of the experimental data (SpaceStation) with ten point-to-point and 12
point-to-line correspondences.
}
\label{fig:Resgistration}
\end{figure*}

\newpage


\section{Observations on Global Convergence and Convexity}
\label{sec:ConvexityFigures}
In this section, we provide numerical results that suggest the proposed method is also
applicable to 3D registration problems and may achieve global convergence. Additionally,
we show more figures about the potential convexity of \eqref{eq:minS}.


\subsection{Toward Global Convergence of the Proposed Method}
We show that the proposed method in the experimental data collected by the \velodyne 
reduces the P2P distance by 68.6\%, as shown in Fig.~\ref{fig:SystematicError}.
In addition, we illustrate the proposed algorithm can be used in 3D
registration problems (point-to-point, point-to-line, point-to-plane) by scaling the
simulation (named Random) and experimental data (named RubikCube and SpaceStation) in
\cite{olsson2008solving, briales2017convex}. The results of the registration problems
are shown in Fig.~\ref{fig:Resgistration} and the experiment setup from
\cite{olsson2008solving} is shown in Fig.~\ref{fig:SpaceStationExpSetup}.

\begin{figure}[b]%
\centering
\includegraphics[width=0.81\columnwidth, trim={0.0cm 0cm 0cm 0cm},clip]{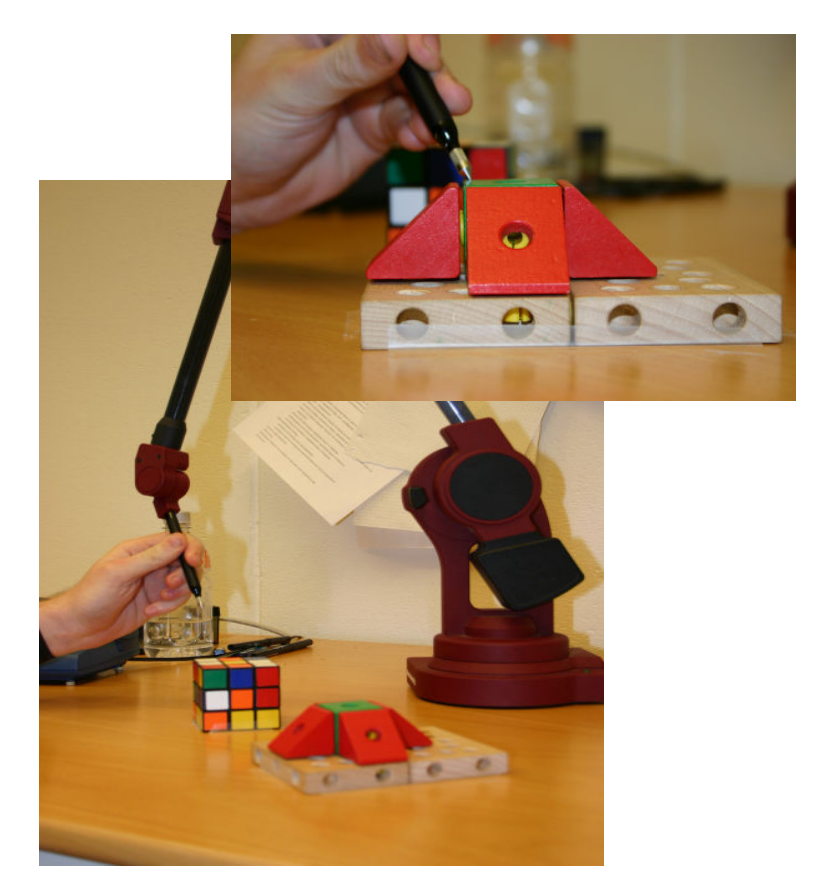}
\caption{This figure is taken from \cite{olsson2008solving} to illustrate the
experimental setup.} 
\label{fig:SpaceStationExpSetup}%
\end{figure}

\newpage
\subsection{Potential Convexity of $f(s)$ and Bisection Method}
Theoretically, a dense search on the scaling $s$ should be performed on a range of
interest. However, we have empirically observed that \eqref{eq:minS} is convex for all of
our data sets.
Figure~\ref{fig:convexIntrinsicAll} shows $f(s)$ is convex for all the sets of the
calibration parameters for our \lidar intrinsic calibration datasets. Similarly, in
Fig.~\ref{fig:Convex3DResgistration} shows the convex property of $f(s)$ for the
simulation and experimental data in \cite{olsson2008solving, briales2017convex}. We, therefore, utilize the bisection method
to determine the scaling $s$. It is emphasized that the convexity of
$\eqref{eq:minS}$ is unknown at this time.

Let $\pre[k][L]s$ and $\pre[k][U]s$ be the lower and upper bounds
of the scaling parameters at the $k$-th iteration. Let $\pre[k+1]s =
(\pre[k][U]s + \pre[k][L]s)/2$ be the update for $(k+1)$-th iteration. The
updates of the scaling parameters are defined as:
\begin{equation}
\begin{aligned}
    \label{eq:UpdateScaling}
    \pre[k+1][L]s &=
\begin{cases}
    \pre[k+1]s, \text{~if~}\nabla f(\pre[k+1][]s) < -\epsilon\\
\pre[k][L]s, \text{~~~~~~otherwise}
\end{cases}\\
    \pre[k+1][U]s &=
\begin{cases}
    \pre[k+1]s, \text{~if~}\nabla f(\pre[k+1][]s) > \epsilon \\
\pre[k][U]s, \text{~~~~~~otherwise}
\end{cases},
\end{aligned}
\end{equation}
where the symmetric difference is used to approximate $\nabla f(\pre[k+1][]s)$:
\begin{equation}
    \label{eq:gradApprox}
    \nabla f(\pre[k+1]s)\approx \frac{f(\pre[k+1]s+h) - f(\pre[k+1]s-h)}{2h},
\end{equation}
where $h>0$ is a small value (taken as $10^{-3}$).

\begin{remark}
    To obtain more accurate results, $\R(\pre[k+1]s + h),~\v(\pre[k+1]s + h),~\R(\pre[k+1]s - h)$ and $\v(\pre[k+1]s - h)$ should be computed separately.
\end{remark}





\begin{figure}[b]%
\centering
\subfloat[]{%
\label{fig:convexIntrinsicAll}%
\includegraphics[width=0.95\columnwidth, trim={0cm 0cm 0cm 0cm},clip]{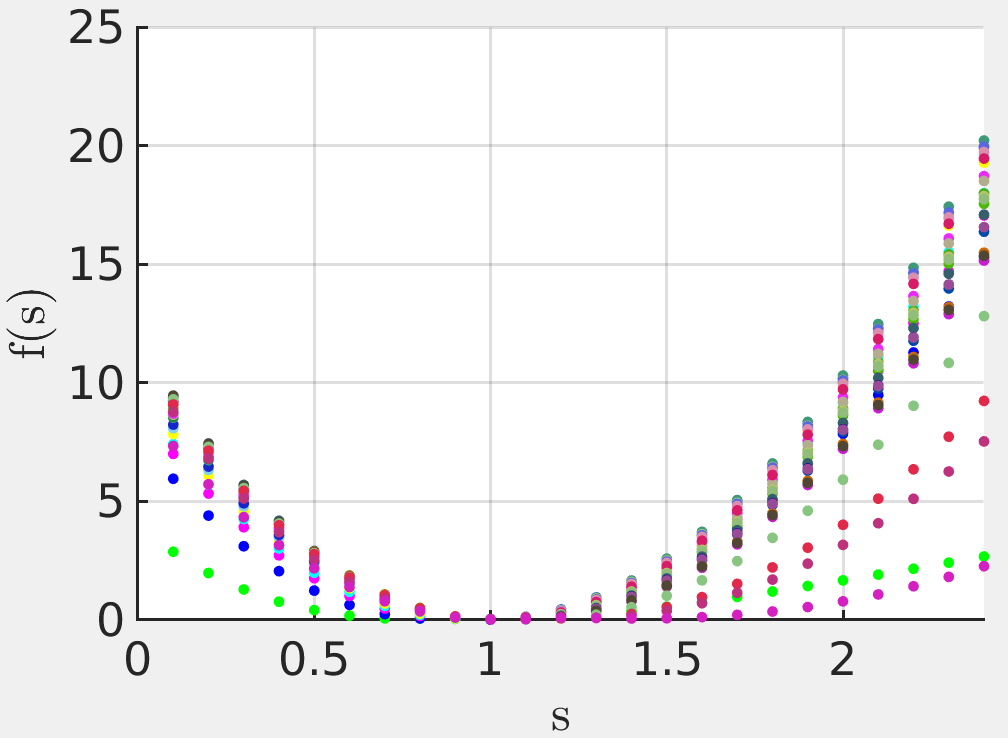}}\\
\subfloat[]{%
\label{fig:Convex3DResgistration}%
\includegraphics[width=0.95\columnwidth, trim={0cm 0cm 0cm 0cm},clip]{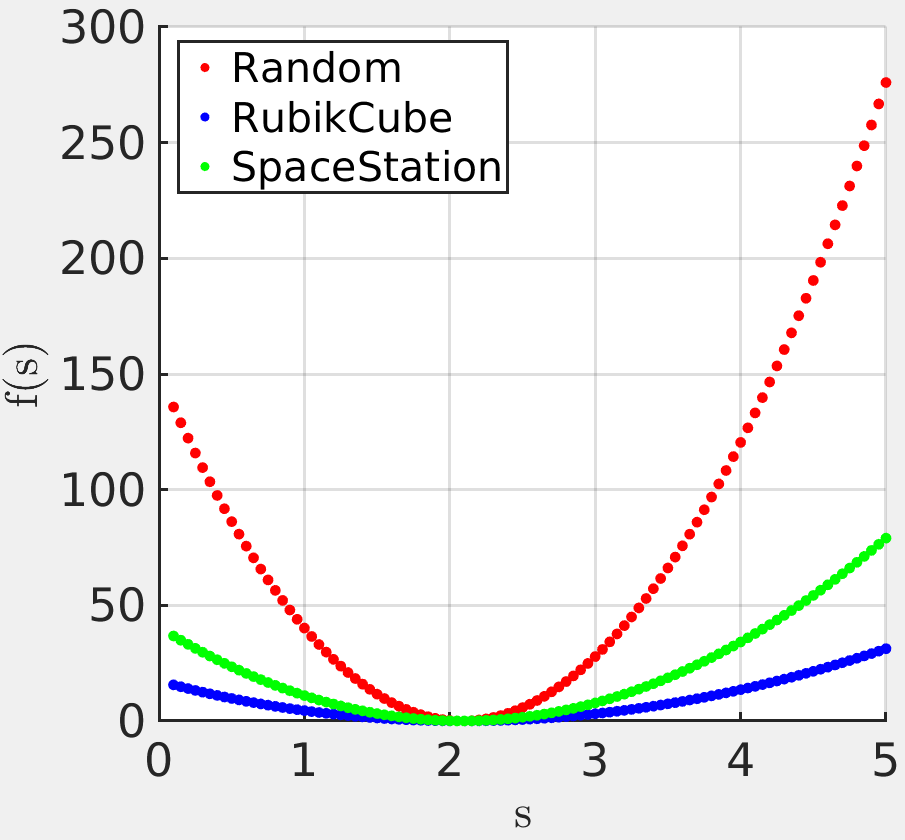}}
\caption[]{
    This figure shows $f$ vs $s$ for the calibration parameters. The top figure shows
    the calibration parameters for all the rings, for $f$ defined in \eqref{eq:minS}. 
    The bottom figure show $f$ for the 3D registration problems (point-to-point,
    point-to-line, point-to-plane). We suspect that the convex shape seen in the plot
is true in general. On going work is seeking a proof. 
}
\label{fig:ConvexALL}
\end{figure}

\end{appendices}

\end{document}